\documentclass{article}

\usepackage{microtype}
\usepackage{graphicx}
\usepackage{booktabs} 
\usepackage{lipsum}
\usepackage{subcaption}

\usepackage{hyperref}

\usepackage[accepted]{icml2023}

\usepackage{amsmath}
\usepackage{amssymb}
\usepackage{mathtools}
\usepackage{amsthm}
\usepackage{bm}
\usepackage{bbm}

\usepackage[capitalize,noabbrev]{cleveref}

\theoremstyle{plain}
\newtheorem{theorem}{Theorem}[section]
\newtheorem{proposition}[theorem]{Proposition}

\newtheorem{corollary}[theorem]{Corollary}
\theoremstyle{definition}

\theoremstyle{remark}
\newtheorem{remark}[theorem]{Remark}

\DeclareMathOperator*{\argmin}{arg\,min}

\DeclareMathOperator{\ddiv}{div}

\DeclareMathOperator{\proj}{proj}
\DeclareMathOperator{\refl}{refl}

\newcommand{\R}{\mathbb{R}}
\newcommand{\Z}{\mathbb{Z}}

\newcommand{\E}{\mathbb{E}}
\newcommand{\paren}[1]{\left(#1\right)}

\newcommand{\sqbrac}[1]{\left[#1\right]}

\newcommand{\norm}[1]{\left\|#1\right\|}
\newcommand{\grad}{\nabla}
\newcommand{\parderiv}[2]{\frac{\partial #1}{\partial #2}}
\newcommand{\inn}[1]{\left\langle#1\right\rangle}
\renewcommand{\vec}{\mathbf}

\newcommand{\dd}{\mathrm{d}}

\newcommand\blfootnote[1]{%
  \begingroup
  \renewcommand\thefootnote{}\footnote{#1}%
  \addtocounter{footnote}{-1}%
  \endgroup
}

\usepackage[textsize=tiny]{todonotes}

\icmltitlerunning{Reflected Diffusion Models}

\begin{document}

\twocolumn[
\icmltitle{Reflected Diffusion Models}



\icmlsetsymbol{equal}{*}

\begin{icmlauthorlist}
\icmlauthor{Aaron Lou}{stanford}
\icmlauthor{Stefano Ermon}{stanford}
\end{icmlauthorlist}

\icmlaffiliation{stanford}{Department of Computer Science, Stanford University}

\icmlcorrespondingauthor{Aaron Lou}{aaronlou@stanford.edu}

\icmlkeywords{Diffusion Models}

\vskip 0.3in
]



\printAffiliationsAndNotice{}  

\begin{abstract}
    Score-based diffusion models learn to reverse a stochastic differential equation that maps data to noise. However, for complex tasks, numerical error can compound and result in highly unnatural samples. Previous work mitigates this drift with thresholding, which projects to the natural data domain (such as pixel space for images) after each diffusion step, but this leads to a mismatch between the training and generative processes. To incorporate data constraints in a principled manner, we present Reflected Diffusion Models, which instead reverse a reflected stochastic differential equation evolving on the support of the data. Our approach learns the perturbed score function through a generalized score matching loss and extends key components of standard diffusion models including diffusion guidance, likelihood-based training, and ODE sampling. We also bridge the theoretical gap with thresholding: such schemes are just discretizations of reflected SDEs. On standard image benchmarks, our method is competitive with or surpasses the state of the art without architectural modifications and, for classifier-free guidance, our approach enables fast exact sampling with ODEs and produces more faithful samples under high guidance weight. 

\vspace{-0.6cm}
\end{abstract}

\begin{figure*}[ht!]
    \centering
    \includegraphics[width=0.95\textwidth]{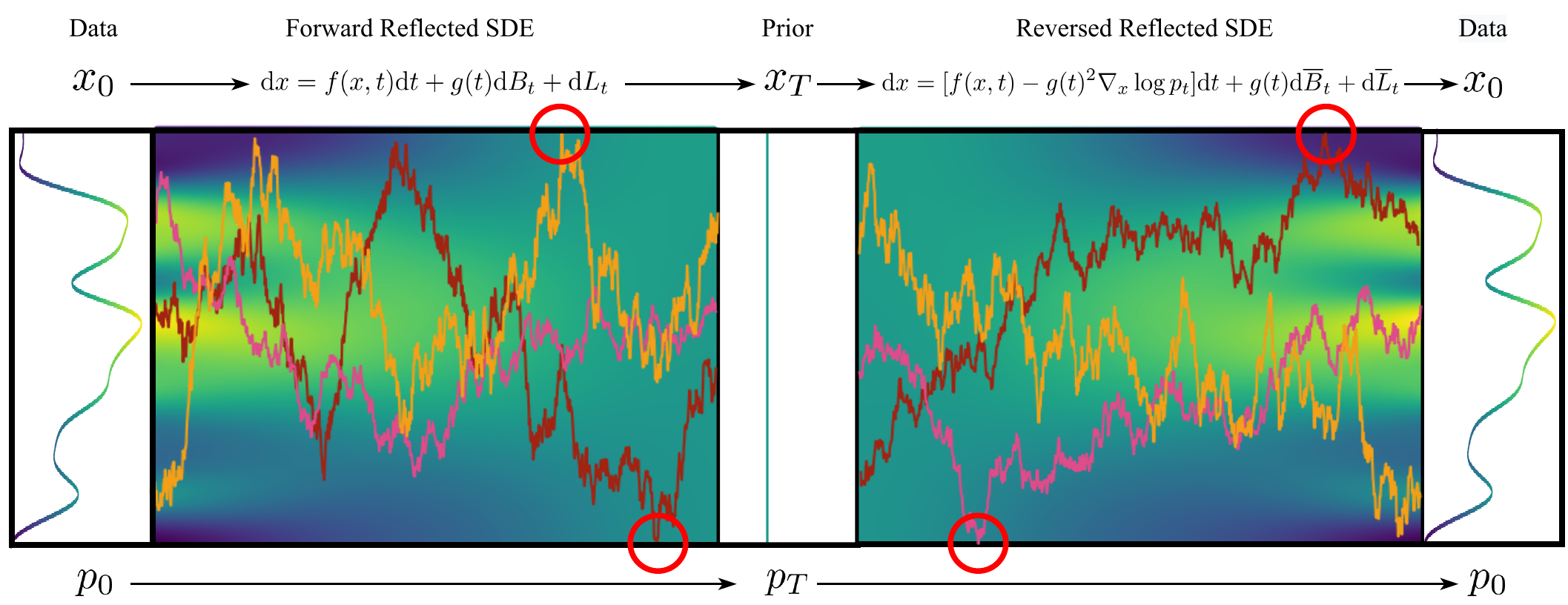}
    \caption{\textbf{Overview of Reflected Diffusion Models.} We map a data distribution $p_0$ supported on $\Omega$ to the prior distribution $p_T$ through a reflected stochastic differential equation (Section \ref{sec:method:rsde}). Whenever a Brownian trajectory hits $\partial \Omega$, it is reflected back in instead of escaping (circled in red), so $p_t$ is supported on $\Omega$ for all $t$. We can recover $p_0$ from $p_T$ with a reversed reflected stochastic differential equation (Section \ref{sec:method:reverse}) by learning the Stein score $\grad_x \log p_t$ (Section \ref{sec:scorematching}). Our generative model is guaranteed to be constrained in $\Omega$.
    }\label{fig:main}
\end{figure*}

\section{Introduction}

Originally introduced in \citet{SohlDickstein2015DeepUL} and later augmented in \citet{song2019generative,Ho2020DenoisingDP, Song2020ScoreBasedGM}, diffusion models have quickly become one of the most ubiquitous deep generative models, with applications in many domains including images \citep{Dhariwal2021DiffusionMB}, natural language \citep{Li2022DiffusionLMIC}, and molecule generation \citep{Xu2022GeoDiffAG}. Additionally, their stability and scabality have enabled the deployment of large text-to-image systems \citep{Ramesh2022HierarchicalTI}.\blfootnote{Code Link: \href{https://github.com/louaaron/Reflected-Diffusion/}{https://github.com/louaaron/Reflected-Diffusion/}}

Diffusion models learn to reverse a stochastic process that maps data to noise, but, as a result of inherent approximation error of both the SDE discretization and score matching, they often follow incorrect trajectories. This behavior compounds error, so models can diverge and generate highly unnatural samples on more complex tasks (see for instance Figure \ref{fig:noclip}). To mitigate this degeneration, many diffusion models modify the sampling process by projecting to the support of the data after each diffusion step \citep{Ho2020DenoisingDP, Li2022DiffusionLMIC}, a technique known as thresholding. This incorporates the known constraints of the data distribution, stabilizing sampling and avoiding divergent behavior. Notably, this oft-overlooked detail underlies many pixel-based image diffusion models \citep{Ho2020DenoisingDP, Dhariwal2021DiffusionMB} (appearing as a $[0, 255]$ clipping function at each diffusion step) and is essential for text-to-image generation \citep{Saharia2022PhotorealisticTD}. Although thresholding avoids failure, it is theoretically unprincipled because it leads to a mismatch between the training and generative processes. Furthermore, this mismatch can introduce artifacts such as oversaturation \citep{Ho2022ClassifierFreeDG} that necessitate further modifications \citep{Saharia2022PhotorealisticTD}.

In this work, we present Reflected Diffusion Models, a class of diffusion models that, by design, respects the known support of the data distribution. Unlike standard diffusion models, which perturbs the data density with Brownian motion, our method evolves the distribution with reflected Brownian motion that always stays within the boundary. We then parameterize the reversed diffusion process with the scores of the perturbed density, which we learn using a new score matching method on bounded domains. The resulting generative model is a reflected SDE that automatically incorporates the data constraints without altering the generative process. We provide an overview of our method in Figure \ref{fig:main}.

Our proposed methodology has several merits:

\textbf{Scales to high dimensions.} To learn the score function on a general bounded domain, we introduce constrained denoising score matching (CDSM). Unlike previous methods \cite{Hyvrinen2007SomeEO}, CDSM scales to high dimensions, and we develop an algorithm for fast computation. Since reflection operations are negligible compared to neural network computation, our training and inference times are effectively equivalent to those of standard diffusion models.

\textbf{Key features transfer over.} We show that ODE sampling \citep{Song2020ScoreBasedGM}, diffusion guidance \citep{Ho2022ClassifierFreeDG}, and maximum likelihood bounds \citep{Song2021MaximumLT} extend to the reflected setting. As such, our method can be modularly applied to preexisting diffusion model systems.

\textbf{Justifying and correcting previous methods.} We draw connections with the thresholding methods used in pixel-space diffusion models \citep{Saharia2022PhotorealisticTD}. These methods all sample from a reflected stochastic differential equation despite being trained on a standard diffusion process. Correctly training with our CDSM loss avoids pathological behavior and allows for equivalent ODE sampling.

\textbf{Broad Applicability.} We apply our method to high-dimensional simplices (e.g. class probabilities) and hypercubes (e.g. images). Using a synthetic example, we show that our method is the a simplex diffusion model that scales to high dimensions. On common image generation benchmarks, our results are competitive with or surpass the current state of the art. In particular, on unconditional CIFAR-10 generation \citep{Krizhevsky2009LearningML}, we achieve a state of the art Inception Score of 10.46 and a comparable FID score of 2.72. For likelihood estimation, our method achieves a second best score of 2.68 and 3.74 bits per dimension on CIFAR-10 and ImageNet32 \citep{Oord2016PixelRN} without relying on either importance sampling or learned noise schedules.
\section{Background}

To introduce diffusion models (in the continuous time formalism of \cite{Song2020ScoreBasedGM}), we first transform a data density $q_0$ on $\R^d$ by applying a ``forward" diffusion process. This takes the form of perturbing points $\vec{x}_0 \sim q_0$ with an SDE with a fixed drift coefficient $\vec{f}: \R^d \times \R \to \R^d$, diffusion coefficient $g: \R \to \R$, and Brownian motion $\vec{B}_t$:
\begin{equation}\label{eqn:forwardsde}
    \dd \vec{x}_t = \vec{f}(\vec{x}_t, t) \dd t + g(t) \dd \vec{B}_t
\end{equation}
The resulting family of time varied distributions $x_t \sim q_t$ approaches a known prior distribution $q_T \approx \mathcal{N}(0, \sigma_T^2 I)$. This density evolution process can be reversed by perturbing samples $x_T \sim q_T$ with a reversed SDE \citep{Anderson1982ReversetimeDE}:
\begin{equation}\label{eqn:backwardsde}
    \dd \vec{x}_t = (\vec{f}(\vec{x}_t, t) - g^2(t) \grad_x \log q_t(\vec{x}_t)) \dd t + g(t) \dd \overline{\vec{B}}_t
\end{equation}
where $\overline{\vec{B}}_t$ is time reversed Brownian motion. Diffusion models approximate this reverse process by learning $\grad_x \log q_t$, known as the score function, through a $\lambda$-weighted score matching loss:
\begin{equation}\label{eqn:weightedscorematching}
    \E_{t, \vec{x}_t \sim q_t} \lambda_t \norm{\vec{s}_\theta(\vec{x}_t, t) - \grad_x \log q_t(\vec{x}_t)}^2
\end{equation}
which most commonly takes the form of the more tractable denoising score matching loss \citep{Vincent2011ACB}:
\begin{equation}\label{eqn:weighteddenoisescorematching}
    \hspace*{-1.75mm}\E_{t, \vec{x}_0 \sim q_0, \vec{x}_t \sim q_t (\cdot | \vec{x}_0)}\lambda_t \norm{\vec{s}_\theta(\vec{x}_t, t) - \grad_x \log q_t(\vec{x}_t | \vec{x}_0)}^2
\end{equation}
Here, $q_t(\vec{x}_t | \vec{x}_0)$ is the transition kernel induced by the SDE in Equation \ref{eqn:forwardsde}. With a learned score $\vec{s}_\theta(\vec{x}, t) \approx \grad_x \log q_t$, one can define a generative model by first sampling $\vec{y}_T \sim \mathcal{N}(0, \sigma_T^2 I)$ and then solving the reverse SDE
\begin{equation}\label{eqn:diffmodelsde}
    \dd \vec{y}_t = (\vec{f}(\vec{y}_t, t) - g(t)^2 \vec{s}_\theta(\vec{y}_t, t)) \dd t + g(t) \dd \overline{\vec{B}}_t
\end{equation}
from time $T$ to $0$, giving an approximate sample from $q_0$.

Diffusion models enjoy many special properties. For example, for certain $\lambda_t$, Equation \ref{eqn:weighteddenoisescorematching} can be reformulated as an ELBO using Girsanov's theorem \citep{Song2021MaximumLT,Kingma2021VariationalDM, Huang2021AVP}, allowing for maximum likelihood training. Furthermore, one can derive an equivalent Neural ODE that can be used for sampling and exact likelihood evaluation \citep{Chen2018NeuralOD}.

\textbf{Guidance.} One can also control the diffusion model to sample from a synthetic distribution $\tilde{q}_t(\vec{x}_t | c) \propto q_t(c | \vec{x_t})^w q_t(\vec{x}_t)$. Here, $c$ is a desired condition such as a class or text description, and interpolating the guidance weight $w$ controls the fidelity and diversity of the samples. This requires the score
\begin{equation}\label{eqn:bayesscore}
    \begin{gathered}
        \grad_x \log \tilde{q}_t(\vec{x}_t | c) = w\grad_x \log q_t(c | \vec{x}_t) + \grad_x \log q_t(\vec{x}_t)
    \end{gathered}
\end{equation}
which can be learned $\tilde{\vec{s}}_\theta(\vec{x}_t, t, c) \approx \grad_x \log \tilde{q}_t(\vec{x}_t | c)$ without requiring explicit training on $\tilde{q}_t$. For example, classifier guided methods \citep{Song2020ScoreBasedGM, Dhariwal2021DiffusionMB} combine a pretrained score function $\vec{s}_\theta(\vec{x}_t, t)$ and classifier $q_t(c | \vec{x}_t)$:
\begin{equation}\label{eqn:classguid}
    \tilde{\vec{s}}_\theta(\vec{x}_t, t, c) := w \grad_x \log q_t(c | \vec{x}_t) + \vec{s}_\theta(\vec{x}_t, t)
\end{equation}
and classifier-free guidance methods \citep{Ho2022ClassifierFreeDG} uses a $c$-conditioned score function and an implicit Bayes classifier $q_t(c | \vec{x}_t) = \frac{q_t(\vec{x}_t | c)}{q_t(\vec{x}_t) q_t(c)}$
\begin{equation}\label{eqn:classfreeguid}
    \tilde{\vec{s}}_\theta(\vec{x}_t, t, c) := (w + 1) \vec{s}_\theta(\vec{x}_t, t, c) - w \vec{s}_\theta(\vec{x}_t, t)
\end{equation}
\textbf{Thresholding.} However, since $\vec{s}_\theta$ is not a perfect score function, there is a mismatch between the modeled backward process and the true forward process. Thus, diffusion models can push a sample to areas where $q_t(\vec{x}_t)$ is small, which creates a negative feedback loop since score matching struggles in low probability areas \citep{Song2019GenerativeMB, Koehler2022StatisticalEO}. This causes the sampling process to diverge and commonly occurs for more complex tasks, especially those involving diffusion guidance.

To combat this, many previous works alter the diffusion sampling procedure using thresholding \citep{Saharia2022PhotorealisticTD, Li2022DiffusionLMIC}, which stabilizes the sampling process with inductive biases from the data. In particular, thresholding applies an operator $\mathcal{O}$ that projects back to the data domain $\Omega$ during each discretized SDE step:
\begin{equation}\label{eqn:thresholding}
    \vec{y}_{t - \Delta t} = \mathcal{O}(\vec{y}_t - \sqbrac{\vec{f}(\overline{\vec{y}}_t, t) - g(t)^2 \vec{s}_\theta(\vec{y}_t, t)} \Delta t) + g(t) \vec{B}_{\Delta t}
\end{equation}
For the case of images, $\mathcal{O}$ can be static thresholding, which clips each dimension to the pixel range $[0, 255]$, and dynamic thresholding, which first normalizes all pixels by the $p$-th percentile pixel before clipping \citep{Saharia2022PhotorealisticTD}.

Thresholding alleviates divergent sampling but comes with considerable downsides. For example, it breaks the theoretical setup since the generative model no longer approximates the reverse diffusion process. This mismatch induces artifacts during sampling and precludes the use of ODE sampling \citep{Song2020ScoreBasedGM}.

\section{Reflected Diffusion Models}\label{sec:method}

In this section, we present Reflected Diffusion Models. These define a generative model on a data domain $\Omega$ (assumed to be connected and compact with nonempty interior and uniform Hausdorff dimension) which outer-bounds the support of the data distribution $p_0$. Our method retains the theoretical underpinnings of diffusion models while incorporating inductive biases from thresholding. We highlight the core mechanisms in Figure \ref{fig:main}.

\subsection{Reflected Stochastic Differential Equations}\label{sec:method:rsde}
To model diffusion processes on a compact domain $\Omega$, we use reflected SDEs. For ease of presentation, we only give an intuitive definition of reflected SDEs and simplify so that $g$ is scalar and $\vec{L}_t$ reflects in the normal direction. In Appendix \ref{sec:app:theory:rsde}, we provide a more rigorous mathematical definition and generalize to matrix diffusion coefficients and oblique reflections. For a full introduction, we recommend the readers consult a monograph such as \citet{Pilipenko2014AnIT}.

Our reflected SDEs perturb an initial datum $\vec{x}_0 \sim p_0$ and are parameterized by a drift coefficient $\vec{f}: \Omega \times \R \to \R^d$ and diffusion coefficient $g: \R \to \R$:
\begin{equation}\label{eqn:reflsde}
    \dd\vec{x}_t = \vec{f}(\vec{x}_t, t)\dd t + g(t) \dd \vec{B}_t + \dd \vec{L}_t
\end{equation}
The first two terms on the right hand side of Equation \ref{eqn:reflsde} are exactly those of Equation \ref{eqn:forwardsde}, showing that our reflected SDE behaves like a regular SDE in the interior of $\Omega$. $\vec{L}_t$ is the additional boundary constraint that, intuitively, forces the particle to stay inside $\Omega$. When $x_t$ hits $\partial \Omega$, $\vec{L}_t$ neutralizes the outward normal-pointing component.

This reflected SDE has a unique strong solution as long as $\vec{f}$ and $g$ are Lipschitz in state and time and $\Omega$ satisfies the uniform exterior sphere condition \citep[Theorem 2.5.4]{Pilipenko2014AnIT}, which ensures that $\partial \Omega$ is sufficiently regular. In particular, the uniform exterior sphere condition holds true when $\partial \Omega$ is smooth and even when $\Omega$ is a convex polytope.

\subsection{Density Evolution and Time Reversal}\label{sec:method:reverse}

When we perturb $p_0$ with the reflected SDE in Equation \ref{eqn:reflsde}, our density evolves according to the Fokker-Planck equation with Neumann boundary condition \citep{Schuss2013BrownianDA}:
\vspace{-1.5mm}
\begin{equation}\label{eqn:neumfp}
    \begin{gathered}
        \parderiv{}{t}p_t = \ddiv(- p_t\vec{f} + \frac{g^2}{2} \grad_x p_t)\\
        (p_t\vec{f} - \frac{g^2}{2} \grad_x p_t) \cdot \vec{n} = 0, \vec{x} \in \partial \Omega, \vec{n} \text{ normal}, t > 0
    \end{gathered}
    \vspace{-1mm}
\end{equation}
In addition to allowing us to characterize the limiting density $p_T$, this induces a reversed reflected stochastic differential equation \citep{Cattiaux1988TimeRO, Williams1988OnTO}:
\begin{equation}\label{eqn:reversereflsde}
    \begin{gathered}
        \dd\vec{x}_t = (\vec{f}(\vec{x}_t, t) - g(t)^2 \grad_x \log p_t(\vec{x}_t)) \dd t\\
        + g(t) \dd \overline{\vec{B}}_t + \dd \overline{\vec{L}}_t
    \end{gathered}
\end{equation}
where $\overline{\vec{L}}_t$ is the reversed boundary condition. For our case, $\overline{\vec{L}}_t$ also reflects in the normal direction.
\begin{remark}
    The reversed reflected SDE closely resembles the reversed standard SDE given in Equation \ref{eqn:backwardsde}. On one hand, this is natural because local dynamics match: when $\Omega = \R^d$, $\vec{L}_t$ disappears since $\vec{x}_t$ can never hit $\partial \R^d = \emptyset$. On the other hand, it is surprising that we can reverse a reflected diffusion process with another reflected diffusion process, something that does not hold in the discrete time case.
\end{remark}

\subsection{Reflected SDEs in Practice}\label{sec:method:examples}

In our experiments, $\Omega$ will be either the unit cube $C_d := \{\vec{x} \in \R^d : 0 \le x_i \le 1\}$ or the unit simplex, which is given by $\Delta_d := \{\vec{x} \in \R^d: \sum_{i = 1}^d x_i = 1, x_i \ge 0\}$. We often find it more convenient to work with the projected simplex $\overline{\Delta}_d := \{\vec{x} \in \R^d: \sum_{i = 1}^{d - 1} x_i \le 1, x_i \ge 0\}$ as it is bounded in $\R^d$ instead of in a hyperplane.

We will diffuse with the  Reflected Variance Exploding SDE (RVE SDE), a generalization of the Variance-Exploding SDE introduced in \citet{Song2020ScoreBasedGM}. A RVE SDE is parameterized by $\sigma_0 \ll \sigma_1$ and is defined for $t \in [0, 1]$ by
\begin{equation}\label{eqn:reflectve}
    \dd \vec{x}_t = \overline{\sigma}_t \dd \vec{B}_t + \dd \vec{L}_t
\end{equation}
where $\overline{\sigma}_t := \sigma_0^{1 - t} \sigma_1^t \sqrt{2 \log\paren{\frac{\sigma_1}{\sigma_0}}}$. The reverse is
\begin{equation}
    \dd \vec{x}_t = -\overline{\sigma}_t^2 \grad_x \log p_t(\vec{x}_t) \dd t + \overline{\sigma}_t \dd \overline{\vec{B}}_t + \dd \overline{\vec{L}}_t
\end{equation}
Note that the RVE SDE corresponds to a time dilated version of reflected Brownian motion: time $t$ of a RVE SDE corresponds to time $\sigma_t$ of reflected Brownian motion, where $\sigma_t := \sigma_0^{1 - t}\sigma_1^t$. As a result of Equation \ref{eqn:neumfp}, $p_0$ evolves under a heat equation with Neumann boundary conditions:
\begin{equation}\label{eqn:neumhe}
    \begin{gathered}
        \parderiv{}{t}p_t = \frac{g(t)^2}{2} \Delta_x p_t \quad \grad_x p_t \cdot \vec{n} = 0 \text{ on } \partial \Omega
    \end{gathered}
\end{equation}
Note that $p_1$ becomes a uniform density over $\Omega$ for large enough $\sigma_1$. To see this, we can draw intuition from physics: heat homogenizes in a closed container.

\section{Score Matching on Bounded Domains}\label{sec:scorematching}

While the reflected SDE framework provides a nice theoretical pathway to construct a reflected diffusion model, it requires one to learn the score function $\vec{s}_\theta \approx \grad_x \log p_t$ on $\Omega$. We minimize the constrained score matching loss:
\begin{equation}\label{eqn:bsm}
    \frac{1}{2} \E_{\vec{x} \sim p}^\Omega \norm{\vec{s}_\theta(\vec{x}) - \grad_x \log p(\vec{x})}^2
\end{equation}
where we omit time-dependence for presentation purposes. Furthermore, $\E^\Omega$ indicates the domain of the expectation (as opposed to $\E$ which is an integral over $\R^d$). This is because $p$ can be discontinuous at $\partial \Omega$ (since it is $0$ outside of $\Omega$ and can be nonzero on $\partial \Omega$), so constraining the integral ensure regularity properties used for theorems (such as Stokes').

In this section, we review previous methods for score matching on bounded domains, discuss their fundamental limitations, and propose constrained denoising score matching to overcome these difficulties. Additionally, for the RVE SDE introduced in Section \ref{sec:method:examples}, we show how to quickly compute the score matching training objective.

\subsection{Pitfalls of Implicit Score Matching}

One may hope to draw inspiration from the standard paradigm, which transforms the score matching integral
\begin{equation}\label{eqn:sm}
    \frac{1}{2}\E_{\vec{x} \sim q} \norm{\vec{s}_\theta(\vec{x}) - \grad_x \log q(\vec{x})}^2
\end{equation}
into the implicit score matching loss \citep{Hyvrinen2005EstimationON}:
\begin{equation}
    \E_{\vec{x} \sim q} \sqbrac{\ddiv(\vec{s}_\theta)(\vec{x}) + \frac{1}{2}\norm{\vec{s}_\theta(\vec{x})}^2}
\end{equation}
This removes the intractable $\grad_x \log q(\vec{x})$, allowing for estimation using Monte Carlo sampling. However, the derivation requires the use of Stokes' theorem; applying Stokes' theorem to Equation \ref{eqn:bsm} would instead result in
\begin{equation}
    \begin{gathered}
        \E_{\vec{x} \sim p}^{\Omega} \sqbrac{\ddiv(\vec{s}_\theta)(\vec{x}) + \frac{1}{2}\norm{s_\theta(\vec{x})}^2}\\
        + \int_{\partial \Omega} p(\vec{x})\inn{\vec{s}_\theta(\vec{x}), \vec{n}(\vec{x})} \dd \vec{x}
    \end{gathered}
\end{equation}
where $\vec{n}(\vec{x})$ is the interior pointing normal vector. Unlike the case of $\Omega = \R^d$, where the second term disappears since $\partial \R^d = \emptyset$, this result is computationally intractable. Thus, previous work instead proposes to re-weight the loss function with a nonnegative function $h$ that vanishes on the boundary \citep{Hyvrinen2007SomeEO, Yu2020GeneralizedSM}, minimizing
\begin{equation}\label{eqn:bsmreg}
    \frac{1}{2}\E_{\vec{x} \sim p}^{\Omega} h(\vec{x})\norm{\vec{s}_\theta(\vec{x}) - \grad_x \log p(\vec{x})}^2
\end{equation}
Since $h$ vanishes on $\partial \Omega$, we can cleanly apply Stokes' theorem and derive a result without a boundary term, giving an implicit score matching loss:
\begin{equation}\label{eqn:bsmregimp}
    \E_{\vec{x} \sim p}^{\Omega}\sqbrac{\ddiv(h \cdot \vec{s}_\theta)(\vec{x}) + \frac{h(\vec{x})}{2}\norm{\vec{s}_\theta(\vec{x})}^2}
\end{equation}
However, this formulation is not suitable for high dimensions, even with fast numerical algorithms for the divergence operator \citep{Hutchinson1989ASE, Song2019SlicedSM}. This is because the loss is downweighted near the boundaries, so, for a fixed budget, the error can become unbounded as $x \to \partial \Omega$. For high dimensions, the space near the boundary becomes an increasingly larger proportion of the total volume\footnote{Consider the case when $\Omega = [0, 1]^d$. For large $d$, almost all the mass is close to the boundary.}, which greatly hampers the sample efficiency of the loss.

\subsection{Constrained Denoising Score Matching}

Inspired by the empirical success of denoising score matching \citep{Vincent2011ACB, Song2019GenerativeMB}, we present constrained denoising score matching (CDSM). Crucially, denoising score matching, unlike implicit score matching, can directly generalize to bounded domains due to how it handles discontinuities. This means that, unlike previous methods for constrained score matching, the derivation transfers smoothly. The core mechanism is presented in the following proposition, which we prove in Appendix \ref{sec:app:theory:cdsm}.
\begin{proposition}\label{prop:cdsm}
    Suppose that we perturb an $\Omega$-supported density $a(\vec{x})$ with noise $b(\vec{x} | \cdot)$ (also supported on $\Omega$) to get a new density $b(\vec{x}) := \int_\Omega a(\vec{y}) b(\vec{x} | \vec{y}) \dd \vec{y}$. Then, under suitable regularity conditions for the smoothness of $a$ and $b$, the score matching loss for $b$:
    \begin{equation}\label{eqn:pertbsm}
        \frac{1}{2} \E_{\vec{x} \sim b}^\Omega \norm{\vec{s}_\theta(\vec{x}) - \grad_x \log b(\vec{x})}^2
    \end{equation}
    is equal (up to a constant factor that does not depend on $\vec{s}$) to the CSDM loss:
    \begin{equation}\label{eqn:pertcdsm}
        \frac{1}{2} \E_{\vec{x}_0 \sim a}^\Omega \E_{\vec{x} \sim b(\cdot | \vec{x}_0)}^\Omega \norm{\vec{s}_\theta(\vec{x}) - \grad_x \log b(\vec{x} | \vec{x}_0)}^2
    \end{equation}
\end{proposition}

With the constrained denoising score matching loss, we are then able to define a training objective for reflected diffusion models. In particular, since $p_t(\vec{x})$ is a by definition perturbed density of $p_0(\vec{x})$ with transition kernel $p_t(\vec{x} | \cdot)$, the weighted score score matching loss directly becomes:
\begin{equation}\label{eqn:wcdsm}
    \E_{t, \vec{x}_0 \sim p_0, \vec{x}_t \sim p_t (\cdot | \vec{x}_0)}^\Omega \lambda_t \norm{\vec{s}_\theta(\vec{x}_t, t) - \grad_x \log p_t(\vec{x}_t | \vec{x}_0)}^2
\end{equation}
For our reflected SDE, we will set $\lambda_t \propto g(t)^2$, mirroring previous work and minimizing variance during optimization. Interestingly, as we prove in Section \ref{sec:likelihood}, this corresponds to an ELBO loss when we reverse a RVE SDE.

\subsection{Scaling Score Matching Computation}\label{sec:sm:scaling}

\begin{figure*}[h]
    \centering
    \begin{subfigure}[b]{0.24\textwidth}
        \includegraphics[width=\textwidth]{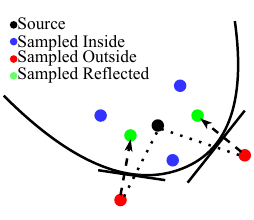}
        \caption*{(i)}
    \end{subfigure}
    \hfill
    \begin{subfigure}[b]{0.24\textwidth}
        \includegraphics[width=\textwidth]{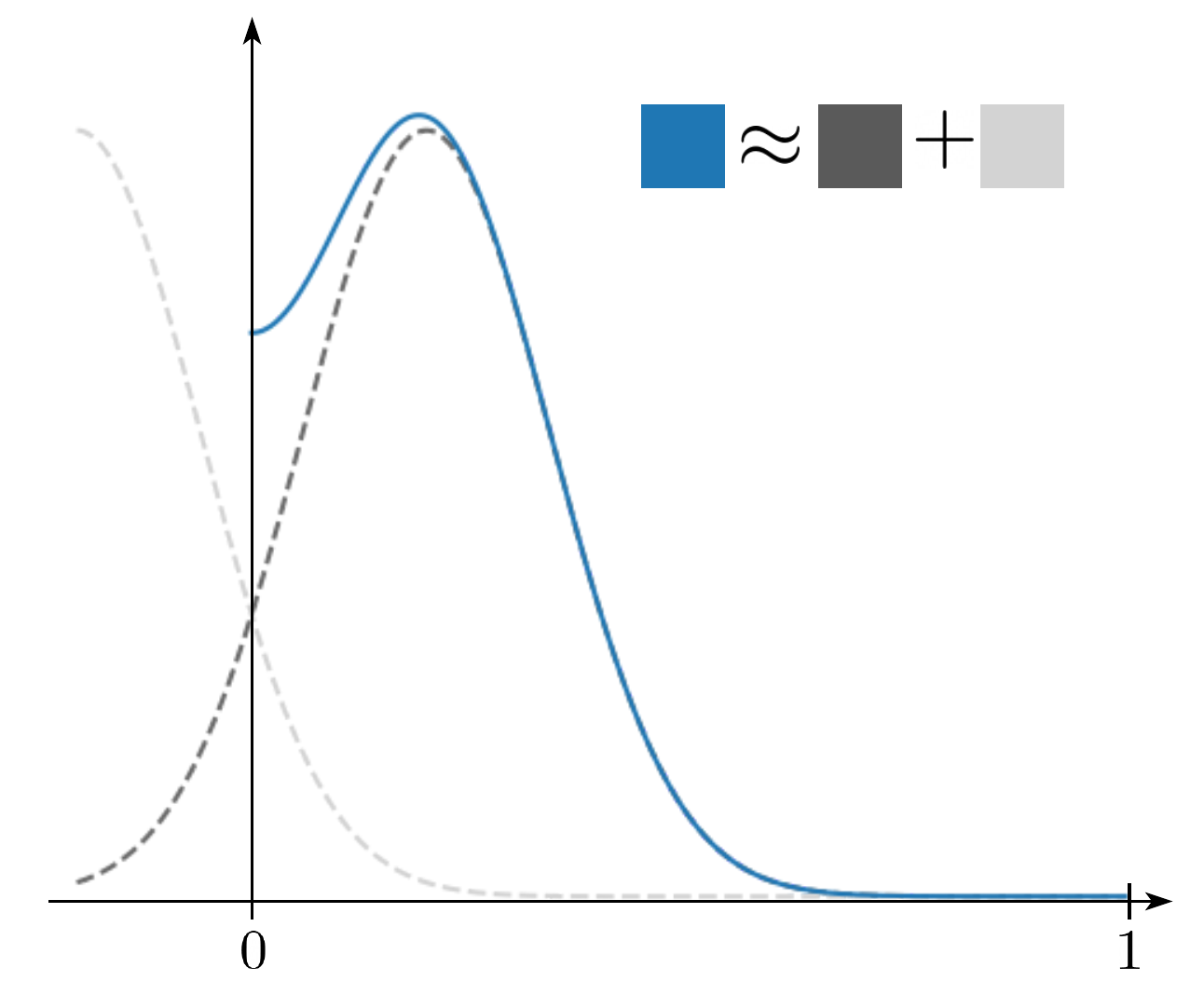}
        \caption*{(ii)}
    \end{subfigure}
    \hfill
    \begin{subfigure}[b]{0.24\textwidth}
        \includegraphics[width=\textwidth]{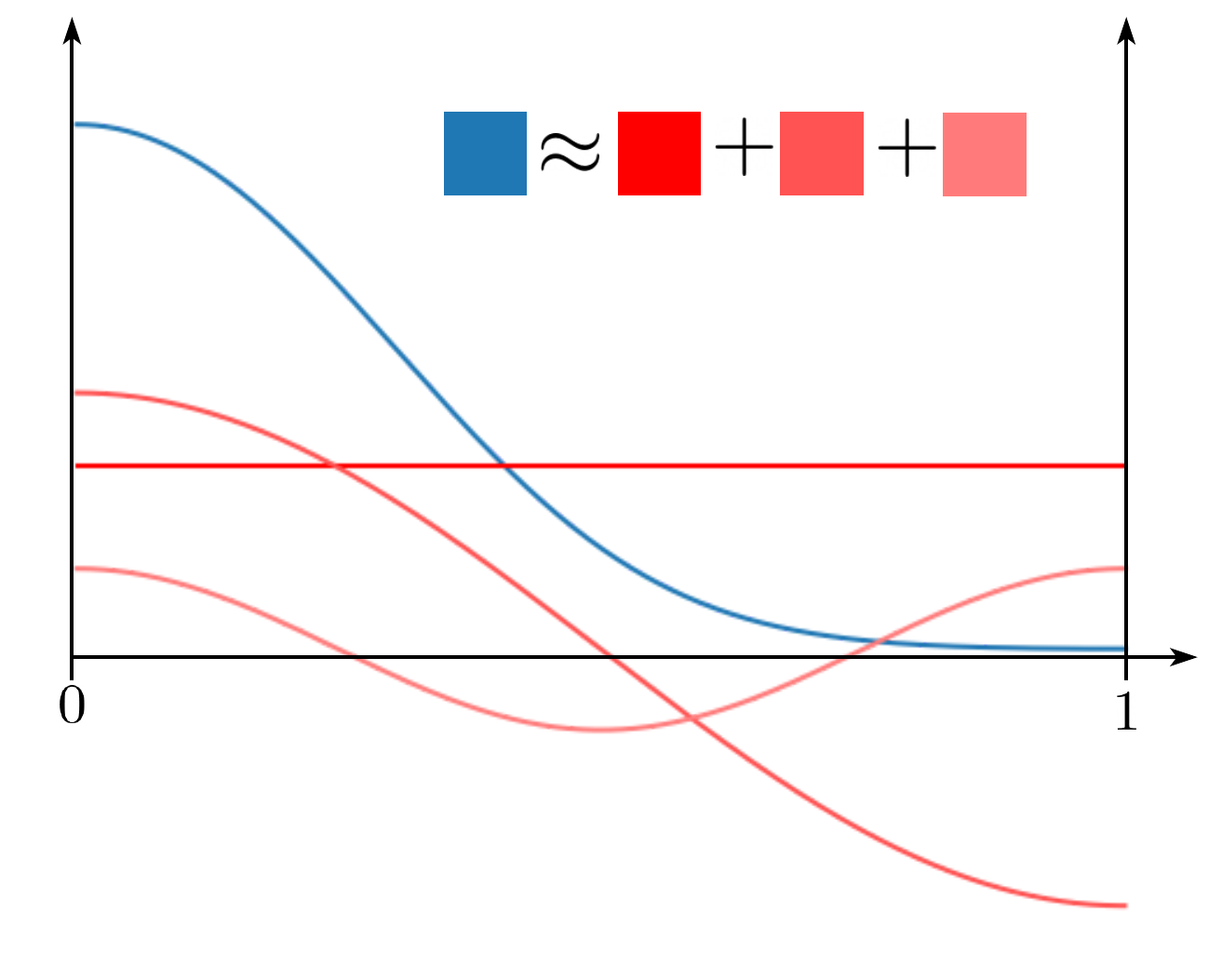}
        \caption*{(iii)}
    \end{subfigure}
    \hfill
    \begin{subfigure}[b]{0.24\textwidth}
        \includegraphics[width=\textwidth]{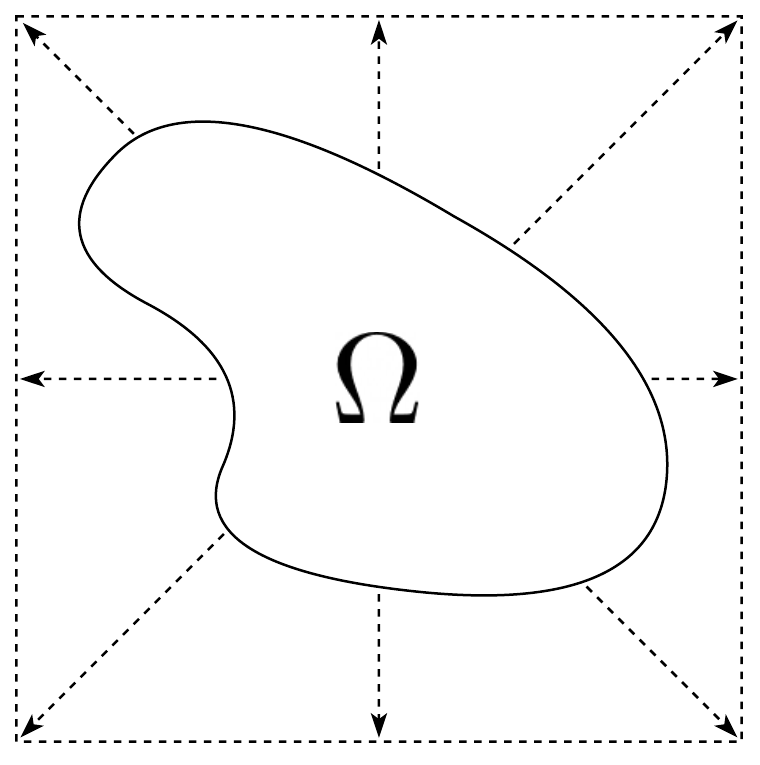}
        \caption*{(iv)}
    \end{subfigure}
    \caption{\textbf{An overview of our computational method for constrained denoising score matching with Brownian transition probabilities.} (i) We can draw samples by sampling $\mathcal{N}(\vec{x}_0, \sigma_t^2 I)$ and then applying reflections on the boundary. (ii) When $t$ is small, we compute the transition density by summing up a mixture of Gaussians (shown for $\Omega = [0, 1]$). (iii) When $t$ is large, we compute using the frequencies of $\Omega$ (shown for $\Omega = [0, 1]$). (iv) We diffeomorphically transform $\Omega \to [0, 1]^d$, where the transition score is tractable.}\label{fig:rhk}
\end{figure*}

We finalize by showing how to sample from and compute the score of the transition density $p_t(\vec{x}_t | \vec{x}_0)$ for the RVESDE. Note that this is the transition density of a reflected Brownian Motion \citep{Harrison1981ReflectedBM}. We highlight the key features of our method in Figure \ref{fig:rhk}.

\textbf{Sampling.} To sample from $p_t(\vec{x}_t | \vec{x}_0)$, we can repeatedly reflect a sample $\vec{y}$ from $\mathcal{N}(\vec{x}_0, \frac{\sigma_t^2}{2} I)$. In particular, we follow the line segment $t \to t\vec{y} + (1 - t)\vec{x}_0$, reflecting in the normal direction when it crosses $\partial \Omega$ and repeating until we reach $t = 1$. This works because, intuitively, the boundary redirects the the Brownian motion but does not change the magnitude. In practice, this process can be quickly computed with classic computational geometric techniques.

\textbf{Score Computation.} There are two approaches for computing the score of $p_t(\vec{x}_t | \vec{x}_0)$ on general geometric domains:

\textit{Approximation with Sum of Gaussians} \citep{Jing2022TorsionalDF}. This method decomposes $p_t(\vec{x}_t | \vec{x}_0)$ into an infinite sum of Gaussian densities that depend on $\vec{x}_0$, $\vec{x}_t$, and the geometry of the domain. For our bounded $\Omega$ with the reflection condition, this gives us the equation
\begin{equation}
    p_t(\vec{x}_t | \vec{x}_0) = \sum_{\vec{x}' \in \mathcal{R}(\vec{x}_t)} p_{\mathcal{N}\paren{\vec{x}_0, \sigma_t^2/2 \cdot I}}(\vec{x}')
\end{equation}
where $p_{\mathcal{N}\paren{\vec{x}_0, \sigma_t^2/2 \cdot I}}$ is the pdf of the Gaussian centered at $\vec{x}_0$ with variance $\sigma_t^2 I$ and $\mathcal{R}(\vec{x}_t)$ is the set of all $\vec{x}' \in \R^d$ s.t. the repeated reflection of the path $t \to t\vec{x}' + (1 - t)\vec{x}_0$ ends in $\vec{x}_t$. Note that this reflection scheme is the same one we use for sampling. Furthermore, through elementary derivations, this gives us a formula for the score $\grad_x \log p_t(\vec{x}_t | \vec{x_0})$.

Generally, this method works quite well for small $\sigma_t$, as we only need to take a small number of local reflections to approximate $p_t(\vec{x}_t | \vec{x}_0)$. However, for larger $\sigma_t$, we need to take many more reflections since the underlying Gaussian is too dispersed, greatly increasing the computational cost.

\textit{Approximation with Laplacian Eigenfunctions} \citep{Bortoli2022RiemannianSG}. This method instead computes using Laplacian Eigenfunctions, a standard technique for solving the heat equation \citep{evans10}. For our problem, these are a (known for each $\Omega$) set of functions $f_i \in L^2(\Omega), i \in \mathbb{N}$ that satisfy $\Delta f_i = -\lambda_i f_i$ and $\grad f_i \cdot \vec{n} = 0$ on $\partial \Omega$. In particular, these form an orthonormal basis for $L^2(\Omega)$, allowing us to solve Equation \ref{eqn:neumhe} directly for an initial density of $\delta_{\vec{x}_0}$:
\begin{equation}
    p_t(\vec{x}_t | \vec{x}_0) = \sum_{i = 0}^\infty e^{-\lambda_i \sigma_t^2 / 2} f_i(\vec{x}_t) f_i(\vec{x}_0)
\end{equation}
This method works well for large $\sigma_t$ because this means that $e^{-\lambda_i \sigma_t^2 / 2} \to 0$, removing the need to evaluate many of the terms. However, for small $\sigma_t$, this method becomes costly because it requires many more terms. Similar to the above method, we can derive a formula for $\grad_x \log p_t(\vec{x}_t | \vec{x_0})$ through this sum.

\textit{Our Method.} We instead propose to combine the above two approaches. In particular, we note that they complement each other: Gaussian sum is accurate for small $\sigma_t$ and eigenfunction sum is accurate for large $\sigma_t$. We can therefore set a $\sigma' \in (\sigma_0, \sigma_1)$ and compute with Gaussian sum when $\sigma_t < \sigma'$ and with eigenfunction sum when $\sigma_t > \sigma'$. In practice, this allows us to upper-bound the number of reflections/eigenfunctions used to $\approx 5$, much fewer than the exponential amount required for each method individually. 

\textit{Scaling to High Dimensions.} By itself, this branching method is unfortunately not enough to scale to very high dimensions. In particular, our computation is, in the worst case, $O(d^k)$ where $d$ is the dimension of $\Omega$ and $k$ is the number of reflection steps or the highest eigenfunction frequency. We have only bounded $k$ to something more manageable. 

To overcome this scaling issue, we consider the simple case of the hypercube $[0, 1]^d$. Since our Brownian motion does not have inter-dimensional interactions, reflections do not interact between non-parallel hyperplanes, and Laplacian eigenfunctions factorize by dimension, we can decompose the probability along each component interval:
\begin{equation}
    p_t(\vec{x}_t | \vec{x}_0) = \prod_{i = 1}^d p_t^i(x_t^i | x_0^i) 
\end{equation}
where $x_t^i$ and $x_0^i$ are the $i$-components of $\vec{x}_t$ and $\vec{x}_0$ respectively and $p_t^i$ is the marginal probability on the $i$-th coordinate. Note that the RVE SDE on $[0, 1]^d$ marginalizes to a RVE SDE on $[0, 1]$ for each dimension:
\begin{equation}
    \dd x_t^i = \overline{\sigma}_t \dd B_t^i + \dd L_t^i
\end{equation}
where $B_t^i$ and $L_t^i$ are the Brownian motion and boundary condition (respectively) for dimension $i$. We can therefore compute on each $\Omega_i = [0, 1]$ and combine the results, reducing the cost from $O(d^k)$ to $O(kd)$. Regular score matching is $O(d)$, and since $k$ is small, we can train Reflected Diffusion Models just as quickly as regular diffusion models.

For more general domains, under certain conditions, we can smoothly and bijectively map from $\mathrm{int}(\Omega) \to (0, 1)^d$. Thus, we can instead learn a diffusion model on $[0, 1]^d$ and then project back to $\Omega$. More details are given in Appendix \ref{sec:app:prac:diffmap}. In particular, this mapping procedure allows us to learn a diffusion model on high-dimensional simplices $\Delta_d$.
\section{Simulating Reflected SDEs}

Combining a score $\vec{s}_\theta$ learned through CSDM and the reverse reflected SDE, we have a Reflected Diffusion Model: sample $\vec{x}_T \sim \mathcal{U}(\Omega)$ and solve the reflected SDE:
\begin{equation}\label{eqn:refldiffmodel}
    \dd \vec{x}_t = -\overline{\sigma}_t^2 \vec{s}_\theta(\vec{x}_t, t) \dd t+ \overline{\sigma}_t \dd \overline{\vec{B}}_t + \dd \overline{\vec{L}}_t
\end{equation}
In this section, we examine numerical methods for simulating examples from this reflected SDE.

\subsection{Euler-Maruyama Discretizations and Thresholding}

The typical Euler-Maruyama discretization of a standard SDE (Equation \ref{eqn:forwardsde}) is given by
\begin{equation}\label{eqn:euma}
    \vec{x}_{t + \Delta t} = \vec{x}_t + \vec{f}(\vec{x}_t, t) \Delta t + g(t) \vec{B}_{\Delta t}
\end{equation}
where $\vec{B}_{\Delta t} \sim \mathcal{N}(0, \Delta t \cdot I)$. For reflected SDEs, one can adapt this discretization by approximating the effect of $\vec{L}_t$ with some suitable operators $\mathcal{O}$.
\begin{equation}\label{eqn:refleuma_proj}
    \vec{x}_{t + \Delta t} = \mathcal{O}(\vec{x}_t + \vec{f}(\vec{x}_t, t) \Delta t + g(t) \vec{B}_{\Delta t})
\end{equation}
Common examples of $\mathcal{O}$ include the projection operator $\proj(x) = \argmin_{y \in \Omega} d(x, y)$ \citep{Liu1993NumericalAT} or the reflection operator $\refl$ used in Section \ref{sec:sm:scaling} \citep{Schuss2013BrownianDA}. One can see that, as $\Delta t \to 0$, both the projection and reflection schemes converge in distribution. Empirically, we find that reflection generates better samples.

Interestingly, this closely mirrors the thresholding step given in Equation \ref{eqn:thresholding}, with the only difference being the choice of operator $\mathcal{O}$ and whether $\mathcal{O}$ is applied before or after the noise step. This difference disappears when $\Delta t \to 0$:
\begin{proposition}[Thresholding solves a reflected SDE]
    Both types of thresholding solve the reflected SDE (Equation \ref{eqn:reflsde}) as $\Delta t \to 0$ under suitable conditions.
\end{proposition}
The full proposition and proof are given in Appendix \ref{sec:app:theory:thresh}.

\subsection{Predictor Corrector}

We extend the predictor-corrector (PC) framework of \citet{Song2020ScoreBasedGM}, which has been shown to improve results. In particular, our learned scores can be used to augment the sampling procedure using Langevin Dynamics \citep{Song2019GenerativeMB}. However, this requires Langevin dynamics for a constrained domain \citep{Bubeck2015FiniteTimeAO}, which, for the probability $p$, are given by the reflected SDE:
\begin{equation}
    \dd\vec{x}_t = \frac{1}{2} \grad_x \log p(\vec{x}_t) \dd t + \dd \vec{B}_t + \dd \vec{L}_t
\end{equation}
During our reversed diffusion iterations, we can discretize the langevin dynamics using Reflected Euler-Maruyama and apply our learned score $\vec{s}(\cdot, t)$:
\begin{equation}
    \vec{x}_t' = \refl(\vec{x}_t + \frac{\epsilon}{2} \vec{s}_\theta(\vec{x}_t, t) + \sqrt{2\epsilon} \cdot \vec{z}) \quad \vec{z} \sim \mathcal{N}(0, 1)
\end{equation}
In practice, we find that PC sampling with a small signal-to-noise ratio noticeably improves image generation results.

\textbf{CIFAR-10 Quality Results.} With these components, we test our method for image generation on the CIFAR-10 dataset and report Inception Score (IS) \citep{Salimans2016ImprovedTF} and Frechet Inception Distance (FID) \citep{Heusel2017GANsTB} in table \ref{tab:cifar10results}. Our models remain competitive, achieving a SOTA Inception score of 10.42. However, Tweedies' formula does generalize to reflected diffusion \citep{Efron2011TweediesFA} (more details are in Appendix \ref{sec:app:prac:denoise}), so our model generates images with imperceptible noise (on the scale of $1-2$ pixels), which degrades the FID score to 2.72 \citep{JolicoeurMartineau2020AdversarialSM}. Despite this, our samples are diverse and visually indistinguishable (Appendix \ref{sec:app:fig}).

\begin{table}[t]
    \centering
    \begin{tabular}{l c c}
        Model & IS $\uparrow$ & FID $\downarrow$\\
        \hline
        NCSN++ \citep{Song2020ScoreBasedGM} & 9.89 & 2.20\\
        DDPM++ \citep{Song2020ScoreBasedGM} & 9.68 & 2.41\\
        Styleformer \citep{Park2021StyleformerTB} & 9.94 & 2.82\\
        UNCSN++ \citep{Kim2021SoftTA} & 10.11 & --\\
        VitGAN \citep{Lee2021ViTGANTG} & 9.89 & 4.87\\
        Subspace NCSN++ \citep{Jing2022SubspaceDG} & 9.99 & 2.17\\
        EDM \citep{Karras2022ElucidatingTD} & -- & \textbf{1.97}\\
        \hline
        Reflected Diffusion (ours) & \textbf{10.46} & 2.72\\
    \end{tabular}
    \caption{\textbf{CIFAR10-Sample Quality Results.} We test Reflected Diffusion Models on CIFAR-10 Image Generation and report IS and FID scores. Our model is highly competitive, achieving a state of the art-inception score for unconditional generation. However, FID lags behind due to noise (as discussed in Appendix \ref{sec:app:prac:denoise})}.
    \label{tab:cifar10results}
    \vspace{-5mm}
\end{table}

\subsection{Probability Flow ODE}

Similarly to the probability flow ODE derived in \citet{Song2020ScoreBasedGM}, one can construct an equivalent deterministic process for a reflected SDE. Interestingly, doing this removes the boundary reflection term, so our deterministic process is exactly the original probability flow ODE derived in \citet{Song2020ScoreBasedGM}:
\begin{equation}\label{eqn:ode}
    \dd \vec{x} = \sqbrac{\vec{f}(\vec{x}, t) - \frac{1}{2}g(t)^2 \grad_x \log p_t(x)}\dd t
\end{equation}
Crucially, the thresholding effect is maintained due to the Neumann condition for $\grad_x \log p_t$ (Equation \ref{eqn:neumfp} line 2) and can't be replicated for standard diffusion models. We elaborate on this construction, as well as connections with DDIM \citep{Song2020DenoisingDI} in Appendix \ref{sec:app:theory:ode}.
 
\section{Diffusion Guidance}

Both classifier and classifier-free guidance (Equations \ref{eqn:classguid} and \ref{eqn:classfreeguid}) extend to Reflected Diffusion Models by logarithm and gradient rules. Since thresholding is primarily useful for diffusion guidance, we investigate the relationship between thresholding, diffusion guidance, and Reflected Diffusion Models on the relatively simple downsampled 64x64 ImageNet dataset \citep{Russakovsky2014ImageNetLS}.

\textbf{Thresholding is critical.} We corroborate \citet{Saharia2022PhotorealisticTD}, showing that pixel-spaced diffusion guidance requires thresholding. We show this for classifier-free guidance in Figure \ref{fig:noclip}, where even a low weight $w=1$ causes about half of the samples to diverge. For classifier guidance, around $75$\% of samples diverge (Figure \ref{fig:cl_guidance_noclip}).

\begin{figure}[t]
    \centering
    \includegraphics[width=0.4\textwidth]{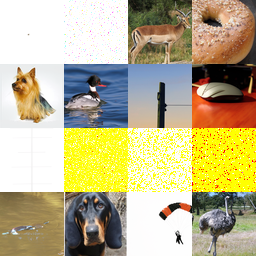}
    \caption{\textbf{Without thresholding, standard diffusion models easily diverge}. We sample using classifier-free guidance ($w=1)$ from a standard diffusion model without using thresholding. Around half of the samples diverge (generating blank images).}\label{fig:noclip}
    \vspace{-5mm}
\end{figure}

\begin{figure*}[!h]
    \centering
    \begin{subfigure}[b]{0.48\textwidth}
        \includegraphics[width=\textwidth]{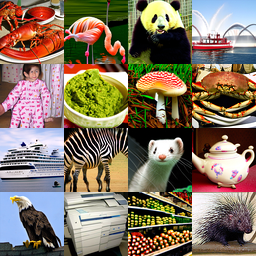}
    \end{subfigure}
    \hfill
    \begin{subfigure}[b]{0.48\textwidth}
        \includegraphics[width=\textwidth]{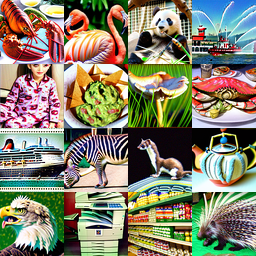}
    \end{subfigure}
    \caption{\textbf{Non cherry-picked guided samples from a reflected and standard diffusion model with high guidance weight.} We compare Reflected Diffusion Models with standard diffusion models for generating class-conditioned 64x64 ImageNet samples for a guidance weight $w=15$. Our generated images are shown on the left, and the baseline is shown on the right (same positions have same classes). Our method retains fidelity while the baseline suffers from oversaturation.}\label{fig:diffguide}
\end{figure*}

\begin{figure}[!h]
    \centering
    \includegraphics[width=0.4\textwidth]{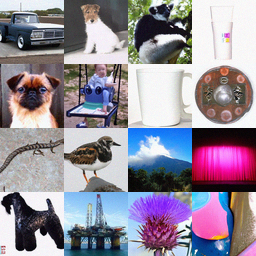}
    \caption{\textbf{Guided ODE samples}. We sample using our ODE with a guidance weight $w=1.5$, retaining image fidelity with fewer forward evaluations (around 100 compared with 1000).} \label{fig:odesample}
    \vspace{-5mm}
\end{figure}

\textbf{Our method retain fidelity under high guidance weight.} Thresholding produces oversaturated images under high guidance weight $w$ \citep{Ho2022ClassifierFreeDG, Saharia2022PhotorealisticTD}, hampering applications which require high fidelity generation. We hypothesize that this is caused by the training and sampling mismatch, and we show in Figure \ref{fig:diffguide} that our method retains fidelity under high guidance weight. We did not find dynamic thresholding method to perform better.

\textbf{ODE sampling works for classifier-free guidance.} The composed score function in classifier-free guidance (Equation \ref{eqn:classfreeguid}) maintains the Neumann boundary condition (Equation \ref{eqn:neumfp}), allowing for ODE sampling. Using this, we demonstrate the first case of high-fidelity classifier-free guided generation using ODEs in Figure \ref{fig:odesample}. Interestingly, ODE equivalent DDIM sampling fails for classifier-free guidance but works for classifier guidance, despite classifier guidance being worse without thresholding (Appendix \ref{sec:app:fig}).

\section{Likelihood Bound}\label{sec:likelihood}

Incidentally, our weighted score matching loss corresponds to an ELBO for our generative model. To show this, we extend Girsanov's Theorem \citep{ksendal1987StochasticDE}, which is used to th derive the ELBO for standard diffusion models \citep{Song2021MaximumLT,Kingma2021VariationalDM,Huang2021AVP}:
\begin{theorem}[Reflected Girsanov for KL divergence]\label{thm:reflgirsanov}
    Suppose we have two reflected SDEs on the same domain $\Omega$
    \begin{align}
        \dd \vec{x}_t &= \vec{f}_1(\vec{x}_t, t)\dd t + g(t) \dd \vec{B}_t + \dd \vec{L}_t\\
        \dd \vec{y}_t &= \vec{f}_2(\vec{y}_t, t)\dd t + g(t) \dd \vec{B}_t + \dd \vec{L}_t
    \end{align}
    from $t = 0$ to $T$ with $\vec{x}_0 = \vec{y}_0 = \vec{z} \in \Omega$. 
    
    Let $\bm{\mu}, \bm{\nu}$ be the path measures for (resp.) $\vec{x}$ and $\vec{y}$. Then,
    \begin{equation}\label{eqn:smgirs}
        \E_{\bm{\mu}}\sqbrac{\log \frac{\dd \bm{\mu}}{\dd \bm{\nu}}} = \frac{1}{2} \int_0^T\E_{p_{\vec{x}_t}(\vec{y})}\sqbrac{g(t)^2 \norm{(\vec{f}_1 - \vec{f}_2)(\vec{y}, t)}^2} \dd t
    \end{equation}
\end{theorem}
The full theorem and proof are given in Appendix \ref{sec:app:theory:girs}.

\begin{table}[!h]
    \centering
    \begin{tabular}{l c c}
        Model & C-10 & IN32\\
        \hline
        \textbf{Non-diffusion} & &\\
        \hline
        Flow++ \citep{Ho2019FlowIF} & 3.08 & --\\
        Pixel-CNN++ \citep{Salimans2017PixelCNNIT} & 2.92 &--\\
        Sparse Transformer \citep{Child2019GeneratingLS} & 2.80 & --\\
        \hline
        \textbf{Diffusion: Modified Noise Schedule} & &\\
        \hline
        ScoreFlow \citep{Song2021MaximumLT}& 2.83 & 3.76\\
        VDM \citep{Kingma2021VariationalDM} & \textbf{2.65} & \textbf{3.72}\\
        \hline
        \textbf{Diffusion: No Noise Modifications} & &\\
        \hline
        ScoreSDE \citep{Song2020ScoreBasedGM} & 2.99 & --\\
        ARDM \citep{Hoogeboom2021AutoregressiveDM} & 2.71 & --\\
        ScoreFlow \citep{Song2021MaximumLT}& 2.86 & 3.83\\
        VDM \citep{Kingma2021VariationalDM} & 2.70 & --\\
        \hline
        Reflected Diffusion (ours) & \textbf{2.68} & \textbf{3.74}\\
    \end{tabular}
    \caption{\textbf{CIFAR-10 and ImageNet32 Bits-per-Dimension (BPD).} No data augmentaiton; lower is Better. We test the likelihood of Reflected Diffusion Models for CIFAR-10 and downsampled ImageNet32 without data augmentation. Our method is second best, nearly matching the state of the art (VDM), without requiring importance sampling or a learned noise schedule.}
    \label{tab:likelihoods}
    \vspace{-4mm}
\end{table}

Note that, by also incorporating the prior and reconstruction loss, Equation \ref{eqn:smgirs} gives us an upper bound on the negative log-likelihood (Appendix \ref{sec:app:theory:girs}). Furthermore, for our reversed RVE SDE, Equation \ref{eqn:smgirs} becomes
\begin{equation}
    \frac{1}{2} \int_0^T \E_{\vec{x}_t \sim p_t}\sqbrac{\overline{\sigma}_t^2 \norm{\vec{s}_\theta(\vec{x}_t, t) - \grad_x \log p_t(\vec{x}_t)}^2} \dd t
\end{equation}
which is a scaled version of our proposed weighted score matching loss in Equation \ref{eqn:wcdsm}. Therefore, we already implicitly train with maximum likelihood. Furthermore, when applied to an individual data point $\vec{x}$, we recover the constrained denoising score matching loss:
\begin{equation}
    \frac{1}{2} \int_0^T \E_{\vec{x}_t \sim p_t(\cdot | \vec{x})}\sqbrac{\overline{g}(t)^2 \norm{\vec{s}_\theta(\vec{x}_t, t) - \grad_x \log p_t(\vec{x}_t | \vec{x})}^2} \dd t
\end{equation}
which allows us to derive an upper bound on $-\log p(\vec{x})$.

\textbf{Image Likelihood Results.} We test Reflected Diffusion Models on CIFAR-10 \citep{Krizhevsky2009LearningML} and ImageNet32 \citep{Oord2016PixelRN} for likelihoods, both without data augmentation. Our method performs comparatively to the SOTA while reducing the number of hyperparameters (in the form of importance sampling and learned noise schedules)\footnote{We omit several results which report a better BPD than VDMs \citep{Kingma2021VariationalDM} on Imagenet32 but a much worse CIFAR-10 result as they test on the the ImageNet32 dataset used for classification \citep{Chrabaszcz2017ADV}, which is significantly easier and incomparable due to the use of anti-aliasing.}. Note that we can compute exact likelihoods through the probability flow ODE, which typically improves results \citep{Song2021MaximumLT}, but, for a fair comparison with VDM, we report the likelihood bound.

\section{Simplex Diffusion}

We also demonstrate that our reflected diffusion model can scale to high dimensional simplices. We train on softmaxed Inception classifier logits for ImageNet \citep{Szegedy2014GoingDW}, which take values in a $1000$-dimensional simplex. Our training dynamics are reported in Figure \ref{fig:simplexresults} (with a $0.99$ EMA), showing that our method is able to optimize the loss (and thus maximize the ELBO) even in high dimensions. Our diffusion process is fundamentally different from the simplex diffusion method from \citet{Richemond2022CategoricalSW}, as we evolve our dynamics directly on the simplex while the previous method diffuses on a higher dimensional space (the positive orthant) and projects to the simplex.

\begin{figure}[t]
    \centering
    \includegraphics[width=0.4\textwidth]{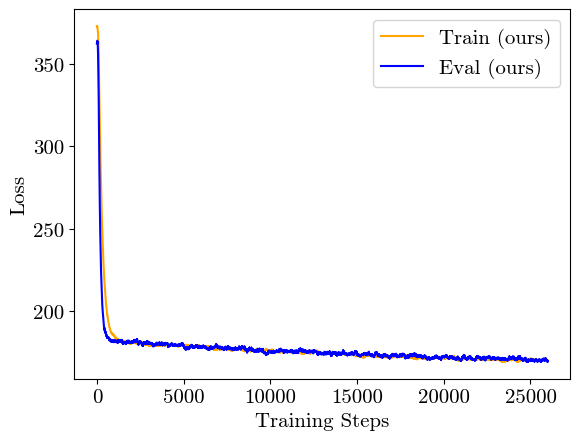}
    \caption{\textbf{Simplex Diffusion Training and Validation Curves.} Our method trains stably in high dimensions.}
    \label{fig:simplexresults}
    \vspace{-5mm}
\end{figure}




\section{Conclusion}

We introduced Reflected Diffusion Models, a diffusion model which respects natural data constraints through reflected SDEs. Our method scales score matching on general bounded geometries and retains theoretical constructs from standard diffusion models. Our analysis also sheds light on the commonly used thresholding sampling method and provides improvements through correct training.

We did not explore architecture or noise scheduling, which are critical for state of the art results; we leave this (and scaling to text to-image-generation) for future work.

Latent Diffusion (LD) \citep{Rombach2021HighResolutionIS} is a diffusion model method that also incidentally does not require thresholding. We hypothesize that this is because both our method and LD directly incorporate data space constraints. Notably, we work over an outer bound of the support of the data distribution, while LD works over a submanifold learned by a VAE. Future work could try to find a middle ground between these two data support approaches.
\section{Acknowledgements}

This project was supported by NSF (\#1651565), ARO (W911NF-21-1-0125), ONR (N00014-23-1-2159), CZ Biohub, and Stanford HAI GCP Grants. AL is supported by a NSF Graduate Research Fellowship. We would also like to thank Chenlin Meng for helpful discussions.

\bibliography{refs}
\bibliographystyle{icml2023}

\newpage
\appendix
\onecolumn
\section{Theoretical Constructs}

\subsection{Reflected Stochastic Differential Equations}\label{sec:app:theory:rsde}

We follow \citet{Pilipenko2014AnIT} in our presentation. Given a domain $\Omega$ and an oblique reflection vector field $\vec{v}$ that satisfies $\vec{v}(\vec{x}) \cdot \vec{n}(\vec{x}) = 1$, where $\vec{n}$ is the inward pointing unit normal vector field, the reflected SDE is defined as
\begin{equation}
    \dd\vec{x}_t = \vec{f}(\vec{x}_t, t)\dd t + \vec{G}(\vec{x}_t, t) \dd \vec{B}_t + \vec{v}(\vec{x}_t) \dd \vec{L}_t
\end{equation}
where $\vec{L}_t$ is defined recursively as $\int_0^t \mathbbm{1}_{\vec{x}_t \in \partial \Omega} \dd \vec{L}_s$. Here, we see that $\vec{L}_t$ is a process that determines whether $\vec{x}_t$ hits the boundary and then applies a reflection. For our purposes in the main paper, $\vec{v} = \vec{n}$ and we surppress the notation for compactness.
Define $\bm{\sigma} = \frac{1}{2} \vec{G}^\top \vec{G}$. When
\begin{equation}
    \vec{v}(\vec{x}_t, t) = \frac{\bm{\sigma}(\vec{x}_t, t) \vec{n}(\vec{x})}{\norm{\bm{\sigma}(\vec{x}_t, t) \vec{n}(\vec{x})}}
\end{equation}
then Equation \ref{eqn:neumfp} generalizes (under suitable regularity conditions) \citep{Schuss2013BrownianDA}:
\begin{equation}\label{eqn:app:neumfp}
    \begin{gathered}
        \parderiv{}{t}p_t(\vec{x}) = \ddiv(- p_t(\vec{x}) \vec{f}(\vec{x}, t) + \bm{\sigma}(\vec{x}, t) \grad_\vec{x} p_t)\\ (p_t(\vec{x}) \vec{f}(\vec{x}, t) - \bm{\sigma}(\vec{x}, t) \grad_\vec{x} p_t) \cdot \vec{n}(\vec{x}) = 0 \text{ when } \vec{x} \in \partial \Omega,  \vec{n} \text{ is normal}, t > 0
    \end{gathered}
\end{equation}
As such, this induces a reverse process \citep{Williams1988OnTO, Cattiaux1988TimeRO} that one can easily check has the same marginal probability distributions:
\begin{equation}
    \dd\overline{\vec{x}}_t = \sqbrac{\vec{f}(\overline{\vec{x}}_t, t) - \bm{\sigma}(\overline{\vec{x}}_t, t) \grad_x \log p_t(\overline{\vec{x}}_t)} \dd t + \vec{G}(\overline{\vec{x}}_t, t) \dd \overline{\vec{B}}_t + \overline{\vec{v}}(\overline{\vec{x}}_t) \dd \overline{\vec{L}}_t
\end{equation}
Here $\overline{\vec{v}}$ is a vector field that satisfies the condition $\overline{\vec{v}} \cdot \vec{n} = 1$ and $\overline{\vec{v}} + \vec{v}$ is a positive multiple of $\vec{n}$.

\subsection{Constrained Denoising Score Matching}\label{sec:app:theory:cdsm}

\begin{proposition}\label{prop:app:cdsm}
    Suppose that we perturb an $\Omega$-supported density $a(\vec{x})$ with noise $b(\vec{x} | \cdot)$ (also supported on $\Omega$) to get a new density $b(\vec{x}) := \int_\Omega a(\vec{y}) b(\vec{x} | \vec{y}) d\vec{y}$. Then, under suitable regularity conditions for the smoothness of $a$ and $b$, the score matching loss for $b$:
    \begin{equation}\label{eqn:app:pertbsm}
        \frac{1}{2} \E_{\vec{x} \sim b}^\Omega \norm{\vec{s}_\theta(\vec{x}) - \grad_x \log b(\vec{x})}^2
    \end{equation}
    is equal (up to a constant factor that does not depend on $\vec{s}$) to the CSDM loss:
    \begin{equation}\label{eqn:app:pertcdsm}
        \frac{1}{2} \E_{\vec{x}_0 \sim a}^\Omega \E_{\vec{x} \sim b(\cdot | \vec{x}_0)}^\Omega \norm{\vec{s}_\theta(\vec{x}) - \grad_x \log b(\vec{x} | \vec{x}_0)}^2
    \end{equation}
\end{proposition}

\begin{proof}
    This proof comes down to showing that

    \begin{equation}
        \E_{\vec{x} \sim b}^\Omega \inn{\vec{s}_\theta(\vec{x}), \grad_x \log b(\vec{x})} = \E_{\vec{x}_0 \sim a}^\Omega \E_{\vec{x} \sim b(\cdot | \vec{x}_0)}^\Omega \inn{\vec{s}_\theta(\vec{x}), \grad_x \log b(\vec{x} | \vec{x}_0)}
    \end{equation}

    which can be done directly

    \begin{align}
        \E_{\vec{x} \sim b}^\Omega \inn{\vec{s}_\theta(\vec{x}), \grad_x \log b(\vec{x})} &= \int_\Omega \inn{\vec{s}_\theta(\vec{x}), \grad_x \log b(\vec{x})} b(\vec{x}) \dd \vec{x}\\
        &= \int_\Omega \inn{\vec{s}_\theta(\vec{x}), \frac{\grad_x b(\vec{x})}{b(\vec{x})} }b(\vec{x}) \dd \vec{x}\\
        &= \int_\Omega \inn{\vec{s}_\theta(\vec{x}), \grad_x b(\vec{x})}\dd \vec{x}\\
        &= \int_\Omega \inn{\vec{s}_\theta(\vec{x}), \grad_x \int_\Omega a(\vec{y}) b(\vec{x} | \vec{y}) \dd \vec{y}}\dd \vec{x}\\
        &= \int_\Omega \inn{\vec{s}_\theta(\vec{x}), \int_\Omega a(\vec{y}) \grad_x b(\vec{x} | \vec{y}) \dd \vec{y}}\dd \vec{x}\\
        &= \int_\Omega \inn{\vec{s}_\theta(\vec{x}), \int_\Omega a(\vec{y}) b(\vec{x} | \vec{y}) \grad_x \log b(\vec{x} | \vec{y}) \dd \vec{y}}\dd \vec{x}\\
        &= \int_\Omega \int_\Omega a(\vec{y}) b(\vec{x} | \vec{y}) \inn{\vec{s}_\theta(\vec{x}), \grad_x \log b(\vec{x} | \vec{y})} \dd\vec{y}\dd \vec{x}\\
        &= \E_{\vec{x}_0 \sim a}^\Omega \E_{\vec{x} \sim b(\cdot | \vec{x}_0)}^\Omega \inn{\vec{s}_\theta(\vec{x}), \grad_x \log b(\vec{x} | \vec{x}_0)}
    \end{align}
\end{proof}

This proof is exactly the same as the one presented in \citep{Vincent2011ACB}. The only difference is that we replace the domain of integration with $\Omega$. Note that the key property that allows us to complete the proof is the convolution identity, which generalizes unlike Stokes' theorem for implicit score matching.

\subsection{Probability Flow ODE and Connections to DDIM}\label{sec:app:theory:ode}

We now derive the probability flow ODE, show how to use it to sample, and discuss connections with DDIM. For convenience, we will work with the assumptions given in the paper (that the diffusion coefficient is a scalar depending on only time and that reflection is in the normal direction), but our results directly generalize (given sufficient regularity conditions) to general noise schedules and oblique reflections.

\begin{proposition}[Probability Flow ODE]
    For the reflected SDE
    \begin{equation}
        \dd\vec{x}_t = \vec{f}(\vec{x}_t, t)\dd t + g(t) \dd \vec{B}_t + \dd \vec{L}_t
    \end{equation}
    The ODE given by
    \begin{equation}
        \dd \vec{x}_t = \sqbrac{\vec{f}(\vec{x}_t, t) - \frac{g(t)^2}{2} \grad_x \log p_t(\vec{x}_t)} \dd t
    \end{equation}
    follows the same probability evolution $p_t$.
\end{proposition}
\begin{proof}
    By the forward Kolmogorov Equation, we ca see that the ODE follows
    \begin{equation}
        \parderiv{}{t}p_t(\vec{x}) = \ddiv(- p_t(\vec{x}) \vec{f}(\vec{x}, t) + \frac{g(t)^2}{2}\grad_\vec{x} p_t)
    \end{equation}
    However, we must confirm that the ODE doesn't exit $\Omega$. By the Neumann boundary conditions for the SDE, we see that 
    \begin{equation}
        (\vec{f}(\vec{x}, t) - \frac{g(t)^2}{2} \grad_x \log p_t(\vec{x})) \cdot \vec{n}(\vec{x})
    \end{equation}
    on the boundary, so the flow induced by the ODE is indeed a valid diffeomorphism from $\Omega \to \Omega$.
\end{proof}
Similar to DDIM, we can derive equivalent processes by annealing the noise.
\begin{proposition}[Annealing Noise Level]
    For the reflected SDE
    \begin{equation}
        \dd\vec{x}_t = \vec{f}(\vec{x}_t, t)\dd t + g(t) \dd \vec{B}_t + \dd \vec{L}_t
    \end{equation}
    The reflected SDE
    \begin{equation}
        \dd\vec{x}_t = \sqbrac{\vec{f}(\vec{x}_t, t) - \frac{g(t)^2 - \overline{g}(t)^2}{2} \grad_x \log p_t(\vec{x}) \dd t} + \overline{g}(t) \dd \vec{B}_t + \dd \vec{L}_t
    \end{equation}
    follows the same probability evoluation $p_t$ for all noise levels $\overline{g} > 0$.
\end{proposition}
\begin{proof}
    This follows directly from our Fokker-Planck Equation.
\end{proof}
\begin{remark}
    In the above proposition, $\overline{g} > 0$ as there is no concept of a reflected ordinary differential equation. However, when the noise is $0$, our limiting process yields an ODE.
\end{remark}
To sample with our score function $\vec{s}_\theta$, we simply solve the reversed process, which is
\begin{equation}\label{eqn:app:odesampler}
    \dd \vec{x}_t = \sqbrac{\vec{f}(\vec{x}_t, t) - \frac{g(t)^2}{2} \vec{s}_\theta(\vec{x}_t, t)} \dd t
\end{equation}
for our ODE and
\begin{equation}\label{eqn:app:ddimsampler}
    \dd\vec{x}_t = \sqbrac{\vec{f}(\vec{x}_t, t) - \frac{g(t)^2 + \overline{g}(t)^2}{2} \vec{s}_\theta(\vec{x}_t, t)} + \overline{g}(t) \dd \overline{\vec{B}}_t + \dd \overline{\vec{L}}_t
\end{equation}
for our annealed reflected SDE.

When training with our CSDM objective $\vec{s}_\theta(\vec{x}_t, t)$, the ODE sampler (Equation \ref{eqn:app:odesampler}) to mimic standard reflected diffusion sampling, which includes thresholding. Conversely, when the score is trained with standard score matching, the sampler just removes thresholding, causing the process to simulate the diffusion path without thresholding.

To mimic the thresholding effect, one must instead turn to the the annealed reflected SDE sampler of Equation \ref{eqn:app:ddimsampler}. If we discretize the equation and (with an abuse of notation) set $\overline{g} = 0$, then we recover the thresholded DDIM sampler. Unfortunately, changing $\overline{g}$ necessarily causes the sampled distribution to shift since $\vec{s}_\theta(\vec{x}_t, t)$ is not trained to mimic the correct $\grad_x \log p_t(\vec{x})$, so the reverse process necessarily results in divergent behavior.

\subsection{Girsanov Theorem for Reflected SDEs and Likelihood Evaluation}\label{sec:app:theory:girs}

We derive our likelihood bounds. We first recall Girsanov's Theorem for SDEs \citep{ksendal1987StochasticDE}

\begin{theorem}[Girsanov Theorem]\label{thm:app:girs}
    Let $\Phi$ be a bounded functional on the space of continuous functions $C([0, T])$. For the SDE evolving on $[0, T]$ with

    \begin{equation}
        \dd \vec{X}_t = \mu(t, \vec{X}_t) \dd t + \sigma(t, \vec{X}_t) \dd \vec{B}_t
    \end{equation}

    we have

    \begin{equation}
        \E \Phi(\vec{X}_t) = \E \sqbrac{\Phi(\vec{B}_t) \exp\paren{-\int_0^T \mu(s, \vec{X}_t) \dd\vec{B}_t - \frac{1}{2} \int_0^T \norm{\mu(t, \vec{X}_t)}^2 \dd t}}
    \end{equation}

    where the expectation is taken is the path measure of the SDE.
\end{theorem}

We then prove the analogue of this for reflected SDEs:

\begin{theorem}[Girsanov Theorem for Reflected SDEs]\label{thm:app:girsrefl}
    Let $\Phi$ be a bounded functional on the space of continuous functions $C([0, T])$. For the reflected SDE evolving on $\Omega$ space and $[0, T]$ time with

    \begin{equation}
        \dd \vec{X}_t = \mu(t, \vec{X}_t) \dd t + \sigma(t, \vec{X}_t) \dd \vec{B}_t + \dd \vec{L}_t
    \end{equation}

    where $\vec{L}_t$ is assumed to have normal reflection. We have

    \begin{equation}
        \E \Phi(\vec{X}_t) = \E \sqbrac{\Phi(\vec{B}_t) \exp\paren{-\int_0^T \mu(s, \vec{X}_t) \dd\vec{B}_t - \frac{1}{2} \int_0^T \norm{\mu(t, \vec{X}_t)}^2 \dd t}}
    \end{equation}

    where the expectation is taken over the path measure of the reflected SDE.
\end{theorem}

\begin{proof}
   We first smoothly extend $\mu$ and $\sigma$ to all of $\R^d$ s.t. the value goes to $0$ very quickly on $\overline{\Omega}$. We then consider the processes $\vec{X}_t^n$ defined $i \in \R^+$ by

    \begin{equation}
        \dd \vec{X}_t^i = \mu_i(t, \vec{X}_t) \dd t + \sigma(t, \vec{X}_t) \dd \vec{B}_t + \dd \vec{L}_t
    \end{equation}

    where $\mu_i(t, x) = \mu(t, x) + i d(x, \Omega) \vec{v}(x)$ where $d$ is the distance function and $\vec{v}(x)$ is the unit normal vector pointing from $x$ to $y$ where $y := \argmin_{z \in \partial \Omega} d(x, z)$. It is well known that $\vec{X}_t^i \to \vec{X}_t$ in measure as $i \to \infty$ \citep{Liu1993NumericalAT}. Since $\Phi$ is a bounded (and thus continuous) functional, we thus have $\E \Phi(\vec{X}_t^i) \to \E \Phi(\vec{X}_t)$ as $i \to \infty$. We finalize by noting that
    \begin{equation}\label{eqn:app:igirs}
        \E \Phi(\vec{X}_t^i) = \E \sqbrac{\Phi(\vec{B}_t) \exp\paren{-\int_0^T \mu_i(s, \vec{X}_t^i) \dd\vec{B}_t - \frac{1}{2} \int_0^T \norm{\mu_i(t, \vec{X}_t^i)}^2 \dd t}}
    \end{equation}
    As $i \to \infty$, $\vec{X}_t^i$ will remain in $\Omega$ w.p. $1$ and $\mu_i = \mu$ on $\Omega$. Therefore, we have the desired convergence
    \begin{equation}
        \E \Phi(\vec{X}_t^i) \to \E \sqbrac{\Phi(\vec{B}_t) \exp\paren{-\int_0^T \mu(s, \vec{X}_t) \dd\vec{B}_t - \frac{1}{2} \int_0^T \norm{\mu(t, \vec{X}_t)}^2 \dd t}}
    \end{equation}
    as desired.
\end{proof}
\begin{corollary}
    As a corollary, when $\Phi$ is $\log \frac{\dd \mu}{\dd \nu}$, this gives us Theorem \ref{thm:reflgirsanov}.
\end{corollary}

\begin{remark}
    It is possible that we can generalize our theorem to obliquely reflected SDEs, although we did not pursue this line of inquiry.
\end{remark}

\begin{remark}
    Theorem \ref{thm:reflgirsanov} recovers the denoising score matching loss if we slice an initial $\delta_x$ distribution. In particular, this is the continuous time ``diffusion loss" $\mathcal{L}_T$ that is used to the form the ELBO for standard diffusion models  \citep{Kingma2021VariationalDM, Ho2020DenoisingDP}.
\end{remark}

\subsection{Thresholding}\label{sec:app:theory:thresh}

On $[-1, 1]^D$, the dynamic thresholding operator is defined by
\begin{equation}
    \mathrm{dynthresh}_p(\vec{x}) = (\proj_{[-1, 1]}(x_i / \max(s, 1))) \quad s \text{ is the } p\text{-th percentile of } |x_i|
\end{equation}
which of course can be scaled to $[0, 1]^D$ (for our setup).
\begin{proposition}[Static Thresholding Solves the Reflected SDE]
    On domains $\Omega$ between times $[0, T]$, the discretization
    \begin{equation}
        x_{t + \Delta t} = \proj(x_t + \vec{f}(x_t, t) \Delta t) + g(t) \vec{B}_{\Delta t}
    \end{equation}
    solves the reflected SDE
    \begin{equation}
        \dd X_t = \vec{f}(X_t, t) + g(t) \dd \vec{B}_t + \vec{n} \dd \vec{L}_t
    \end{equation}
    as $\Delta t \to 0$ when $\vec{f}$ and $g$ are uniformly Lipschitz in time and space and satisfy the linear growth condition for any Lipschitz extension of $\vec{f}$ to the general space $\R^d$.
\end{proposition}

\begin{proof}
    This closely mirrors the proof showing that the standard projection scheme
    \begin{equation}
        y_{t + \Delta t} = \proj(y_t + \vec{f}(x_t, t) \Delta t) + g(t) \vec{B}_{\Delta t}
    \end{equation}
    converges to the solution of the reflected SDE as $\Delta t \to 0$ \citep{Skorokhod1961StochasticEF, Schuss2013BrownianDA, Liu1993NumericalAT}. The key difference is that the process is not supported on $\Omega$ since the projection happens after each. However, since our extension on $\vec{f}$ is Lipschitz, this error is well behaved and disappears as $\Delta t \to 0$.
\end{proof}

\begin{corollary}[Dynamic Thresholding Solves the Reflected SDE]
    With the same conditions as given above, on $[-1, 1]^D$, the discretization
    \begin{equation}
        \vec{x}_{t + \Delta t} = \mathrm{dynthresh}_p(\vec{x}_t + \vec{f}(\vec{x}_t, t) \Delta t) + g(t) \vec{B}_{\Delta t}
    \end{equation}
    converges to the solution of the reflected SDE when $\vec{f}(\vec{x}_t, t)$ does not point outside of $[-1, 1]^D$ on $\ge (1 - p)D$ dimensions.
\end{corollary}
\begin{proof}
    Under our conditions, $\mathrm{dynthresh}_p$ becomes the projection operator since the $p$-th percentile of $|x_i|$ will always be $1$. This replicates the above proposition.
\end{proof}
\begin{remark}
    In practice, we found that learned score networks $\vec{f}$ satisfy the ``pointing inside" condition above. In particular, as $\Delta t \to 0$, $\mathrm{dynthresh}_p$ tends to behave exactly like $\proj$ for all $p < 1$.
\end{remark}
\section{Practical Implementation}

\subsection{Exact Equations for Reflected Brownian Transition Probabilities}

For $[0, 1]$, the reflected transition probability for a source $\vec{x}$ with diffusion value $\sigma$ (which correspond to the mean and standard deviation for the standard normal distribution) is given by
\begin{equation}
    p_{\mathcal{R}(x, \sigma^2)}(y) = \sum_{z: y + z \in \Z} p_{\mathcal{N}(x, \sigma)}(z) = 1 + 2\sum_{k = 1}^\infty e^{-k^2 \pi^2 \sigma^2/2} \cos(k\pi x) \cos(k\pi y)
\end{equation}
Note that this means that the eigenfunctions of $[0, 1]$ under our Neumann boundary condition are $1$ and $\cos(k\pi x)$, with eigenvalues of $0$ and $\pi k^2$

\subsection{Mapping $\mathbf{\Omega}$ to $\mathbf{[0, 1]^d}$}\label{sec:app:prac:diffmap}

Our domain $\Omega$ has an interior which maps bijectively to $(0, 1)^d$ iff $\Omega$ is simply connected. Note that this encompasses a wide variety of domains, notably convex sets.

To construct a map $f: [0, 1]^d \to \overline{\Delta}_t$, we use a variant of the common stick breaking procedure:

\begin{equation}
    (f(\vec{x}))_i = x_i \prod_{j = i + 1}^d (1 - x_i)
\end{equation}

which admits an inverse
\begin{equation}
    (f^{-1}(\vec{y}))_i = \frac{y_i}{1 - \sum_{j = i + 1}^d y_i}
\end{equation}

\subsection{Denoising The Final Proability Distribution} \label{sec:app:prac:denoise}

We note that Tweedies' formula \citep{Efron2011TweediesFA} does not hold for general bounded domains $\Omega$. We show this for $[0, 1]$: given an initial distribution $X$, a perturbed distribution $Y$ constructed by $y \sim \mathcal{R}(x, \sigma^2)$, where $x \sim X$ has a Tweedie denoiser:
\begin{align}
    \E[x | y] &= \int_0^1 x p(x | y) dx\\
    &= \int_0^1 x \frac{p(y | x)p_X(x)}{p_Y(y)} dx\\
    &= \int_0^1 x \frac{p_{\mathcal{R}(x, \sigma^2)}(y) p_X (x)}{p_Y(y)} dy\\
    &\neq u + \sigma^2 \frac{d}{dy} \log p_Y(y)
\end{align}
The reason why this works in the standard case is because the score of the Gaussian distribution is $\frac{y - x}{\sigma}$, which allows us to extract out the desired value. For reflected Gaussians, this does not hold.

Instead, for our experimental results on CIFAR-10, we denoise by training a denoising autoencoder \citep{Vincent2008ExtractingAC} trained on our reflected Gaussian noise. This follows the same architecture as our score network, and is trained to predict the noise (so that subtracting it recovers the initial sample). In general, this is required to get a decent FID score, but makes little difference in terms of perceptual quality as our images are accurate to within a $2.5$ standard deviation noise to begin with.
\section{Experimental Setup}

\subsection{Image Generation}\label{sec:app:experiment:imgqual}

We exactly follow \citet{Song2020ScoreBasedGM} for both models and training hyperparameters. The only differences are that we set $\sigma_1 = 5$ instead of $50$ (for the VE SDE) since we mix well with $\sigma_1 = 5$ while VE SDE needs a much larger $\sigma_1$ to mix. Furthermore, we use the deep DDPM++ architecture, but we rescale the output $\frac{1}{\sigma}$ as is done for the NCSN++ architecture (for VE SDE).

For sampling, we sample with $1000$ predictor (Reflected Euler-Maruyama) steps with $1000$ corrector (Reflected Langevin) steps \citep{Song2020ScoreBasedGM}. We use a signal-to-noise ratio of $0.03$.

\subsection{Image Likelihood}\label{sec:app:experiment:imglike}

We almost exactly follow \citet{Kingma2021VariationalDM} for both models and training hyperparameters, replacing the standard diffusion with our reflected diffusion. We do not train with the noise schedule, instead setting $\sigma_0 = 10^{-4}$ and $\sigma_1 = 5$, which causes the reconstruction and prior losses to be (numerically) $0$.

\subsection{Guided Diffusion}\label{sec:app:experiment:guided}

We exactly follow \citet{Ho2022ClassifierFreeDG} and train using the ADM architecture \citep{Dhariwal2021DiffusionMB} with the same parameters (for standard diffusion). For reflected diffusion, we train with $\sigma_0 = 0.01$ and $\sigma_1 = 5$, following our CIFAR-10 experiments. Furthermore, we scale the output by $1/\sigma$ as the neural network outputs the noise vector and not the score.

For the classifier-free guidance baseline, we retrain a ImageNet64 model following \citet{Ho2022ClassifierFreeDG}. For the classifier guided basleine, we use the pretrained models from \citet{Dhariwal2021DiffusionMB}.

We sample using $1000$ steps each. For our diffusion model, we use reflected Euler Maruyama. For the standard model, we use a standard Euler-Maruyama with thresholding after each step. For ODE sampling, we sample using a RK45 solver \citep{Dormand1980AFO}.

\subsection{Simplex Diffusion}\label{sec:app:experiment:simplex}

We consider class probabilities outputted from the Inceptionv3 ImageNet classifier \citep{Szegedy2015RethinkingTI}. In particular, if $\{X_i\}$ is a set of images and $f$ is a classifier that outputs a $1000$-dimensional vector of class probabilities, then we learn a distribution over $\{L_i := f(X_i)\}$. For learning purposes, we clip this to a value in $\overline{\Delta}_{999}$ and apply the transformation given in Appendix \ref{sec:app:prac:diffmap}.

Our model is a simple MLP autoencoder with $4$ intermediate layers of width $512$. We use the Swish activation \citep{Ramachandran2017SwishAS} and apply LayerNorm \citep{Ba2016LayerN}. We train with Adam \citep{Kingma2014AdamAM} at a $2 \cdot 10^{-4}$ learning rate. We apply an exponential moving average with a rate of $0.9999$ before evaluating/generating our data. We visualize our full training and eval graphs below, as well as some samples taken from our model. Overall, we seem to be able to match the distribution reasonably well.
\begin{figure}[H]
    \centering
    \begin{subfigure}[b]{0.48\textwidth}
        \includegraphics[width=\textwidth]{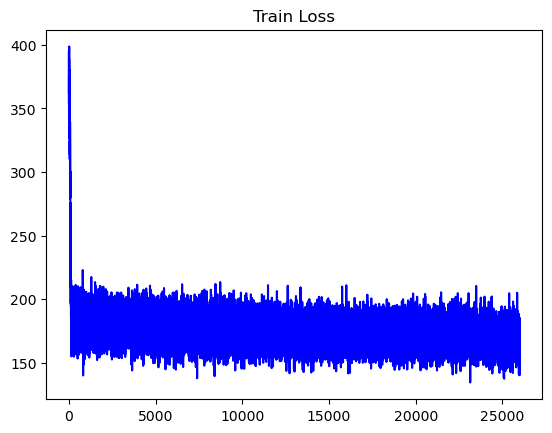}
    \end{subfigure}
    \hfill
    \begin{subfigure}[b]{0.48\textwidth}
        \includegraphics[width=\textwidth]{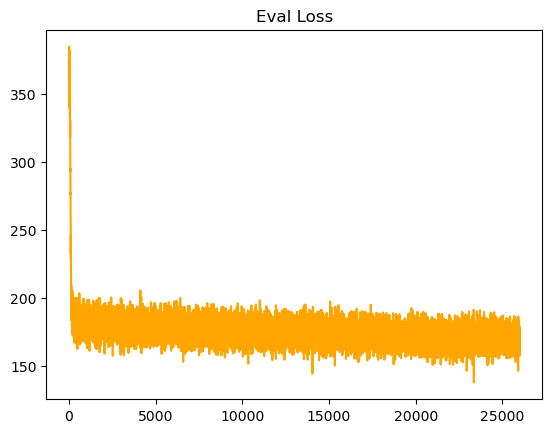}
    \end{subfigure}
    \caption{Training Dynamics for Simplex Diffusion}
\end{figure}
To ensure that we are able to generate data, we generated 10000 and compare the generated histograms of the (most likely) classes. Results are shown below:
\begin{figure}[H]
    \centering
    \begin{subfigure}[b]{0.48\textwidth}
        \includegraphics[width=\textwidth]{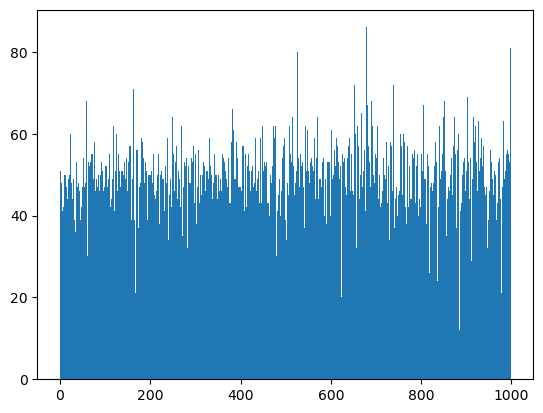}
        \caption{Ground Truth}
    \end{subfigure}
    \hfill
    \begin{subfigure}[b]{0.48\textwidth}
        \includegraphics[width=\textwidth]{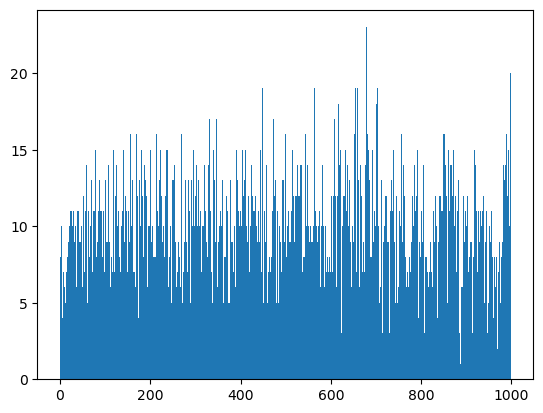}
        \caption{Generated}
    \end{subfigure}
    \caption{Generated Simplex Probabilities for Simplex Diffusion}
\end{figure}

\newpage

\section{Additional Generated Images}\label{sec:app:fig}

\begin{figure}[H]
    \centering
    \includegraphics[width=\textwidth]{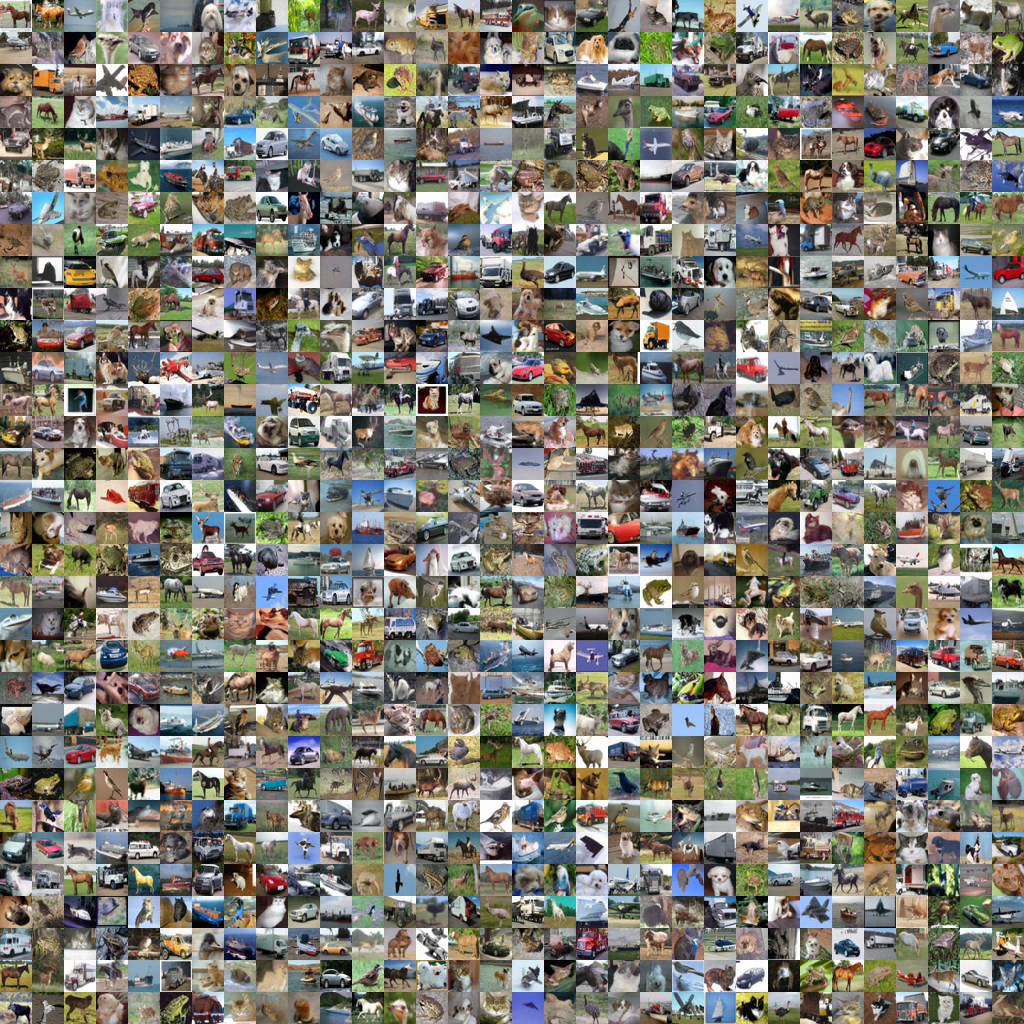}
    \caption{CIFAR-10 Generated Images.}
\end{figure}

\begin{figure}[H]
    \centering
    \includegraphics[width=\textwidth]{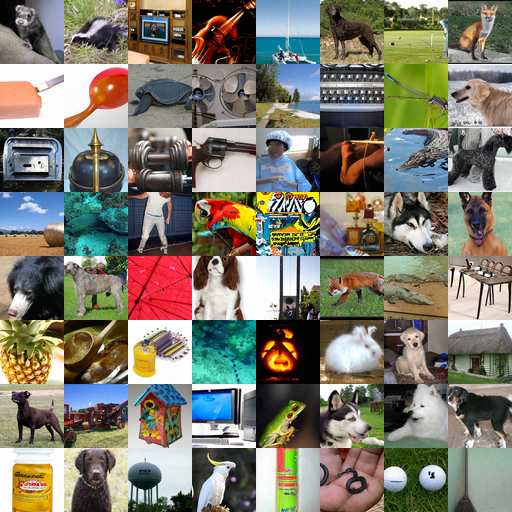}
    \caption{ImageNet64 $w=1$ classifier-free guided samples.}
\end{figure}

\begin{figure}[H]
    \centering
    \includegraphics[width=\textwidth]{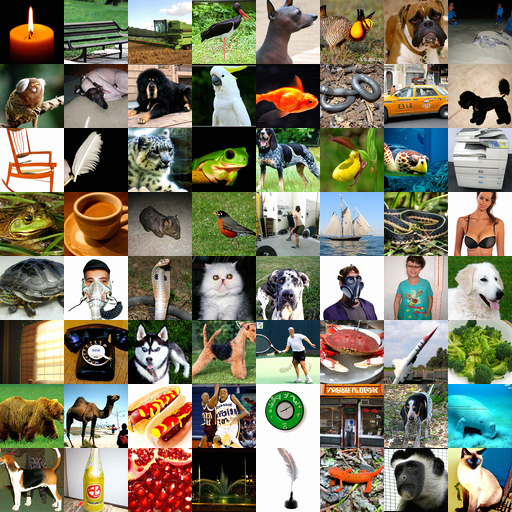}
    \caption{ImageNet64 $w=2.5$ classifier-free guided samples.}
\end{figure}

\begin{figure}[H]
    \centering
    \includegraphics[width=\textwidth]{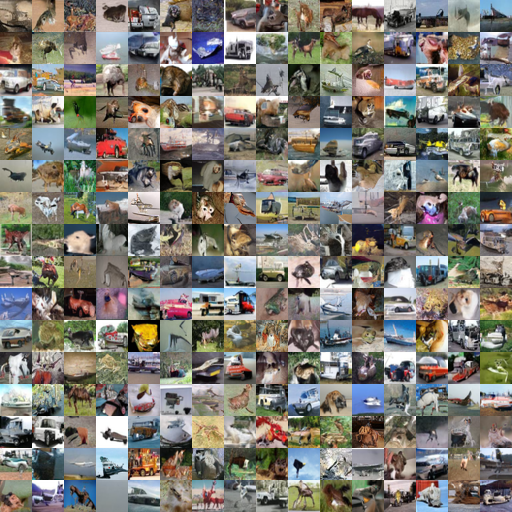}
    \caption{CIFAR-10 Generated Images (trained for BPD).}
\end{figure}

\begin{figure}[H]
    \centering
    \includegraphics[width=\textwidth]{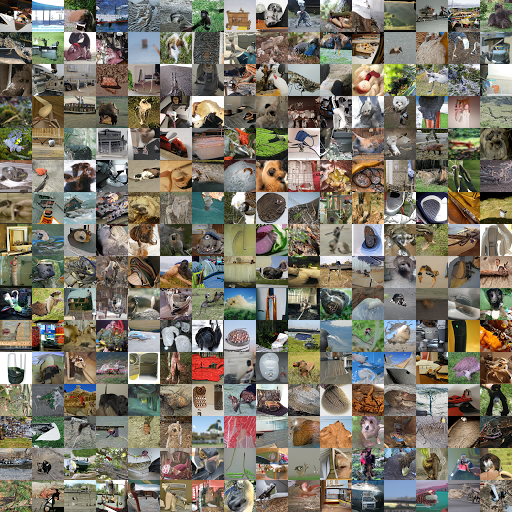}
    \caption{ImageNet32 Generated Images (trained for BPD).}
\end{figure}

\begin{figure}[H]
    \centering
    \includegraphics[width=0.55\textwidth]{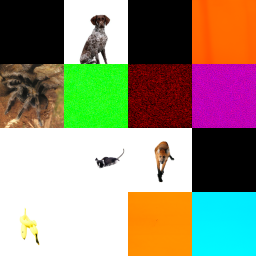}
    \caption{$w=1$ baseline classifier guided images without thresholding. We sample from the pretrained model from \citet{Dhariwal2021DiffusionMB}. Around $75$\% of samples diverge, while most of the rest have noticeable artifacts such as a glaring white background.}\label{fig:cl_guidance_noclip}
\end{figure}

\begin{figure}[H]
    \centering
    \includegraphics[width=0.55\textwidth]{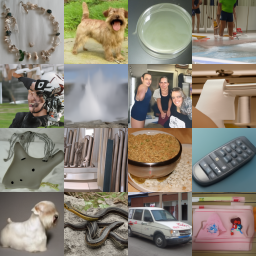}
    \caption{$w=1$ baseline classifier guided images without thresholding and DDIM sampling. Interestingly, these samples don't diverge.}\label{fig:cg_ddim}
\end{figure}

\begin{figure}[H]
    \centering
    \includegraphics[width=0.55\textwidth]{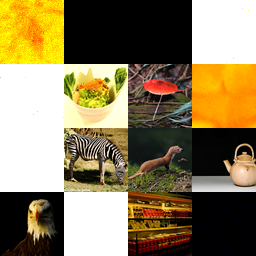}
    \caption{$w=1$ baseline classifier-free guided images sampled with DDIM without thresholding. This corresponds to ODE sampling.}
    \label{fig:app:ode_cf_baseline}
\end{figure}

\begin{figure}[H]
    \centering
    \includegraphics[width=0.55\textwidth]{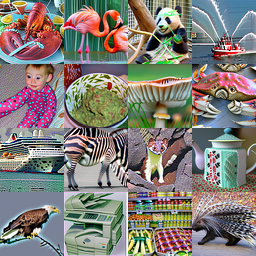}
    \caption{Dynamically thresholded images, matches Figure \ref{fig:diffguide}.}
\end{figure}

\begin{figure}[H]
    \centering
    \includegraphics[width=0.55\textwidth]{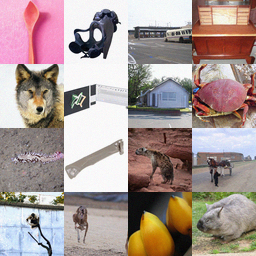}
    \caption{$w=0.5$ ODE samples, Reflected Diffusion.}
    \label{fig:app:0.5ode}
\end{figure}

\begin{figure}[H]
    \centering
    \includegraphics[width=0.7\textwidth]{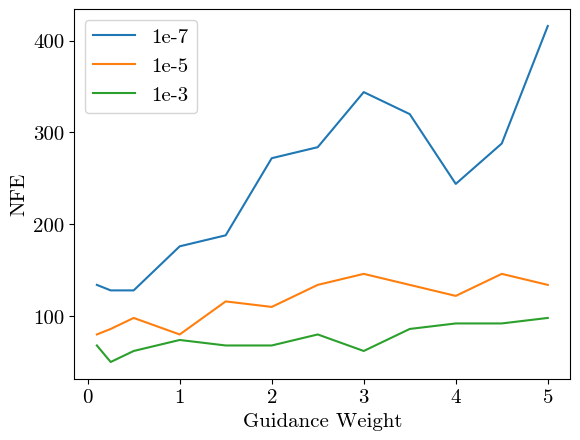}
    \caption{ODE number of forward evaluations (NFE) vs guidance weight for Reflected Diffusion. Increasing the guidance weight tends to increase the number of forward evaluations, but this is still relatively low.}
    \label{fig:app:ode_nfe}
\end{figure}

\end{document}